\documentclass{article}

\usepackage{iclr2027_preprint,times}

\usepackage[utf8]{inputenc} %
\usepackage[T1]{fontenc}    %
\usepackage{hyperref}       %
\usepackage{url}            %
\usepackage{booktabs}       %
\usepackage{amsfonts}       %
\usepackage{nicefrac}       %
\usepackage{graphicx}
\usepackage{microtype}      %
\usepackage{xcolor}         %

\usepackage{microtype}
\usepackage{graphicx}
\usepackage{subcaption}
\usepackage{booktabs} %
\usepackage{hyperref}
\usepackage{algorithm}
\usepackage{algorithmic}
\usepackage{amsmath}
\usepackage{amssymb}
\usepackage{mathtools}
\usepackage{amsthm}
\usepackage{amsthm}
\usepackage{enumitem}
\usepackage{cancel}
\usepackage{algorithm}
\usepackage{comment}

\usepackage[symbol]{footmisc}

\usepackage{enumitem}
\setlist{nosep} 

\theoremstyle{plain}
\newtheorem{theorem}{Theorem}

\newtheorem{lemma}[theorem]{Lemma}

\theoremstyle{definition}

\theoremstyle{remark}

\newcommand{\R}{\mathbb{R}}
\renewcommand{\d}{\mathrm{d}}

\newcommand{\x}{\mathrm{x}}
\renewcommand{\u}{\mathrm{u}}
\newcommand{\hu}{\hat{\mathrm{u}}}

\newcommand{\varphiFun}{\varphi}

\newcommand{\ET}{\mathrm{ET}}
\newcommand{\method}{PaPi}

\title{%
\method{}: Parallel Picard-Duhamel Integration for Fast Neuro-Spectral Architectures}

\author{%
  Arthur Bizzi \\
  Imperial College London \\
  London, United Kingdom \\  
  \texttt{a.bizzi@imperial.ac.uk} \\
  \And
  Márcio Marques \\
  Instituto de Matemática Pura e Aplicada\\
  Rio de Janeiro, RJ, Brazil \\
  \texttt{marcio.marques@impa.br}
  \And
  Kleyton Da Costa \\
  Computer Science \\
  University College London \\
  London, United Kingdom \\  
  \texttt{kleyton.costa.25@ucl.ac.uk} \\
  \And
  Vitor Balestro\\
  Instituto de Matemática e Estatística \\
  Universidade Federal Fluminense\\
  Niterói,  RJ, Brazil \\
  \texttt{vitorbalestro@id.uff.br}
  \And
  Tiago Novello \\
  Instituto de Matemática Pura e Aplicada\\
  Rio de Janeiro, RJ, Brazil \\
  \texttt{tiago.novello@impa.br}
  \And
  João M Pereira\\
  Department of Mathematics \\
  University of Georgia\\
  Athens, GA, USA \\
  \texttt{jpereira@uga.edu}
  \And
  Daniel Yukimura \\
  Instituto de Matemática Pura e Aplicada\\
  Rio de Janeiro, Brazil \\
  \texttt{yukimura@impa.br}
}

\iclrfinalcopy
\begin{document}

\maketitle

\begin{abstract}
Neuro-spectral architectures (NeuSA) are a class of physics-informed neural networks (PINNs) that have proven effective for solving time-dependent PDEs in the frequency domain.
However, training and inference often rely on costly numerical integration of the learned vector field, typically using sequential Runge–Kutta time stepping, which limits temporal parallelism.
We introduce PaPi, an integration scheme inspired by decomposable linear/nonlinear flows that enables 
time-\textbf{Pa}rallel inference through exponential integration and a \textbf{Pi}card-Duhamel scheme.
We show that the Picard-Duhamel iteration converges exponentially for all flows, and provide explicit error bounds that depend on how close the flow is to linear.
Experiments show that in some instances \method{} accelerates training and inference by up to three orders of magnitude, achieving training times of under one second for low-dimensional PDE learning problems, while maintaining competitive predictive accuracy.
\end{abstract}

\section{Introduction}

Physics-informed neural networks (PINNs) approximate solutions of partial differential equations (PDEs) by training neural representations against the governing equations, initial and boundary conditions,
and available data. This formulation supports both forward simulation and inverse problems, including parameter identification and reconstruction from sparse measurements~\citep{raissi2019physics}. PINNs have been applied across science and engineering, including fluid mechanics, transport processes, and more~\citep{karniadakis2021physics,lu2021deepxde}.

For time-dependent problems, PINNs can struggle to capture spatial detail (spectral bias) and respect the temporal causality imposed by the initial condition. Neuro-spectral architectures (NeuSA)~\citep{bizzi-neusa} approach this problem by combining a spatial spectral representation with a system of Neural ODEs~\citep{chen2018neural}. This reformulation allows them to achieve state-of-the-art accuracy with substantially shorter training times.

The network parameterizes the dynamics of the spectral coefficients, and the solution is obtained by integrating these dynamics from the prescribed initial state.
This formulation makes time integration a computational bottleneck. Each training step requires evaluating a trajectory and differentiating a physics-informed loss through the numerical solver. The original NeuSA implementation uses explicit fourth-order Runge--Kutta (RK4) integration, whose time steps are evaluated sequentially. This dependence limits temporal parallelism even when spatial operations are performed on a GPU. Moreover, rapidly oscillating or strongly dissipative linear modes can restrict the stable step size of an explicit solver, increasing the number of steps required for each training update.

We introduce \method{}, a parallel integration and training scheme for neuro-spectral dynamics with a tractable linear component. Rather than propagating the solution sequentially through time, we construct successive approximations to the entire trajectory
using a \emph{Picard--Duhamel} formulation of the underlying ODE. Each iteration incorporates the linear dynamics exactly and updates the nonlinear correction through affine recurrences that can be evaluated using a parallel scan. This allows temporal propagation within each iteration to be carried out in parallel, substantially reducing the sequential depth of the ODE solve.

Our analysis establishes convergence of the continuous
Picard--Duhamel iteration and provides explicit error bounds for the iteration and its temporal discretization. The implementation is tailored to the training of Neuro-Spectral Architectures, reusing trajectories across optimizer steps to reduce the number
of iterations needed during training.

We evaluate \method{} on six forward PDE benchmarks and a
coefficient-identification problem, covering heterogeneous wave propagation, nonlinear transport, and dispersive dynamics. Our experiments demonstrate substantial reductions in training time over the original sequential NeuSA formulation, with competitive
or improved accuracy on most problems. For one-dimensional problems, the proposed models can reach sub-second training times, illustrating how parallel-in-time integration can fundamentally change the computational regime in which PINNs operate.

Our main contribution is to accelerate the training and evaluation of Neuro-Spectral Architectures through parallel time integration. Specifically:

\begin{itemize}

\item We introduce \method{}, a Picard--Duhamel integrator that combines exact linear propagation with parallel updates of the nonlinear correction, reducing the sequential depth of each trajectory update from $O(N)$ to $O(\log N)$.

\item We demonstrate substantial reductions in optimization time across six forward PDE benchmarks and a coefficient-identification problem, achieving sub-second optimization in a one-dimensional benchmark while maintaining competitive accuracy on most problems.

\item We establish convergence of the continuous iteration and derive explicit iteration and quadrature error bounds, clarifying how the nonlinear correction affects the number of iterations required.

\end{itemize}

\section{Related Work}

\textbf{Physics-informed neural networks.}
Early neural methods for differential equations used trial solutions satisfying initial or boundary conditions by design~\citep{lagaris1998artificial}. PINNs popularized flexible neural approximations with physical constraints imposed through residuals evaluated by automatic differentiation~\citep{raissi2019physics}, becoming central to physics-informed machine learning for forward and inverse problems~\citep{karniadakis2021physics,lu2021deepxde}.

Gradient imbalances among PINN loss terms can produce pathological training~\citep{wang2021gradient}, while neural-tangent-kernel analyses reveal large disparities in their convergence rates~\citep{wang2022when}. \citet{krishnapriyan2021characterizing} linked failures on convection-, reaction-, and diffusion-dominated problems to PDE-residual optimization difficulties despite sufficient network expressivity. Remedies include adaptive weighting~\citep{mcclenny2020self}, adaptive sampling and residual refinement~\citep{lu2021deepxde,yu2022gradient}, domain decomposition~\citep{jagtap2020extended,shukla2021parallel}, and architectures tailored to physics-informed optimization~\citep{wang2024piratenets}.

\textbf{Spectral bias and causality.}
Coordinate-based multilayer perceptrons favor low frequencies, hindering high-frequency and multiscale solutions~\citep{tancik2020fourier}. For PINNs, \citet{wang2021eigenvector} analyzed this bias through the neural tangent kernel eigenspectrum and proposed Fourier feature embeddings to improve high-frequency representation. Conventional PINN objectives also do not explicitly enforce temporal causality: \citet{wang2022respecting} showed severe training difficulties on multiscale, chaotic, and turbulent systems and introduced temporally weighted residuals to propagate information progressively from earlier to later times.

NeuSA~\citep{bizzi-neusa} addresses both issues architecturally rather than only adjusting the objective: spatial spectral discretization reduces the PDE to a finite-dimensional dynamical system with a neural vector field integrated as a Neural ODE. The spectral representation exposes spatial modes, and the initial-value formulation imposes temporal causality. Subsequent work has extended NeuSA through temporal-domain decomposition in D-NeuSA, which specializes the learned dynamics across temporal subdomains~\cite{balestro2026neurospectral}, while Perturbative-NeuSA introduces a structured residual formulation that decomposes the solution into a low-fidelity background and a learned perturbation~\cite{zhu2026perturbative}. We retain this representation and address its main computational bottleneck: repeated integration during training.

\textbf{Parallel time integration.}
Parallel-in-time methods address sequential time stepping by exposing temporal concurrency beyond spatial parallelism. Parareal~\citep{lions2001parareal} partitions time and iteratively combines inexpensive coarse propagation with accurate parallel fine solves; \citet{gander2007analysis} analyzed its convergence and connections to multiple shooting and multigrid.

\textbf{Picard iteration for Diffusion Models.}
ParaDiGMS~\citep{shih2023parallel} uses Picard iteration to refine a diffusion model's denoising trajectory, evaluating updates in parallel across time and accumulating them through prefix sums. A sliding temporal window controls memory usage and advances as earlier portions converge, trading additional computation for lower sampling latency.

Our approach uses the integral initial-value formulation: Picard iteration defines a fixed-point problem for the time-integrated vector field, allowing contributions across temporal intervals to be reorganized. For decomposable linear/nonlinear systems, the variation-of-constants (Duhamel) formula separates a tractable linear evolution from the nonlinear remainder, evaluated at the previous iterate. Propagation across the temporal grid then becomes associative composition, enabling parallel scans. We exploit this structure in NeuSA training, where repeated integration of learned high-dimensional spectral dynamics makes temporal parallelism particularly valuable.
 
\section{Method}

\subsection{Background}
We consider the following initial-boundary value problem for $t \in [0,T]$ and $\x \in \Omega \subset \R^d$:
\begin{equation} \label{eq:pde}
    \frac{d}{dt}u(t,x) = \mathbf{F}\left(t,x,u, \nabla u, \nabla\nabla u\right)\,, \quad u(0,x) = u_0(x)\, 
\end{equation}
where  $u: [0,T] \times \Omega \to \mathbb{R}^n$ denotes the (vector-valued) solution, $\nabla u$ and $\nabla\nabla u$ denote its first- and second-order spatial derivatives, and $\mathbf{F}$ is a smooth function. $u_0(\x)$ is the initial condition (Cauchy data), and we assume some boundary conditions are imposed on $\partial\Omega$.

Neuro-spectral architectures describe the time evolution of a PDE as an ODE evolving in a Hilbert space via a spectral decomposition, approximating $u(x,t)$ by its projection onto the Fourier basis \cite{boyd2001spectralmethods,canuto2007spectral,trefethen2000spectral}.
The solution $u$ is represented as a truncated Fourier expansion, stored as the coefficients tensor $\hu(t) : [0,T] \to \mathbb{C}^{c_1\times \dots \times c_d \times n}$, leading to an ODE in the Fourier coefficients themselves:
\begin{eqnarray}\label{eq:spectral-decomp}
   &\mathbf{u}(t,\x) = \sum_k \hu_{k}(t)e^{i \omega_k}, \\ \implies  &\frac{d}{dt}\hat \u = \hat{\mathbf{F}}(\hat \u), \quad \hu(0) = \hu_0,
\end{eqnarray}
where $k$ denotes a $d$-dimensional multi-index, and $\hat{\mathbf{F}}: \mathbb{C}^{c_1\times \dots \times c_d \times n} \to \mathbb{C}^{c_1\times \dots \times c_d \times n}$ denotes the finite-dimensional projection of the operator $\mathrm{F}$ in the Fourier basis.  

Importantly, $\hat {\mathbf{F}}$ does not generally admit a simple closed-form expression in the Fourier basis, and must instead be learned as a parameterized network $\widehat{\mathbf{F}}_\theta$. Nevertheless, it often has structure: a large class of canonical quasi-linear differential operators may be split into homogeneous linear and nonlinear components. In the Fourier basis, the former diagonalizes, leading to:
\begin{equation}
    \hat{\mathbf{F}}(\hu) = M \hu + \varepsilon\tilde{\mathbf{F}}(\hu),
\end{equation}

where $M$ is a diagonal linear operator representing frequencies being multiplied by scalar factors. For example, the Laplacian $\Delta$ reduces to $M\hu = |\omega|^2 \odot \hu$.

NeuSA models the temporal evolution of a PDE solution through a dynamical
system in spectral space. In the original formulation, this system is
integrated using standard time-stepping methods, such as Runge--Kutta
schemes. Since these methods advance the solution sequentially, time
integration dominates the cost of both training and inference.

\subsection{Picard Iterations}

Consider the initial value problem
\begin{equation}
     \frac{\mathrm{d}}{\mathrm{d}t}X(t)=G(t,X(t)),
    \qquad X(0)=X_0,
    \label{eq:general_ode}
\end{equation}
Given a temporal grid
$0=t_0<t_1<\cdots<t_N=T$, a standard time-stepping method introduces a strict dependency of future times on the past, meaning that while operations within each step may be parallelized, the time
steps themselves cannot be evaluated simultaneously.
Integrating \eqref{eq:general_ode} we get:
\begin{equation}
    X(t)=X_0+\int_0^t G(s,X(s))\,\mathrm{d}s.
    \label{eq:integral_ode}
\end{equation}
The solution to \eqref{eq:general_ode} is a fixed point of the integral equation \eqref{eq:integral_ode}, suggesting the Picard iteration
\begin{equation}
    X^{(k+1)}(t)
    =
    X_0+\int_0^t G(s,X^{(k)}(s))\,\mathrm{d}s,
    \qquad k=0,1,\ldots,
    \label{eq:picard_iteration}
\end{equation}
Here, $k$ is an iteration index rather than a time index:
$X^{(k)}$ approximates the entire trajectory on $[0,T]$. Hence, given
$X^{(k)}$, all evaluations $G(t_j,X^{(k)}(t_j))$, $j=0,\ldots,N$, can be performed simultaneously. After temporal discretization, the
remaining prefix integrals reduce to cumulative operations that can be
implemented using a parallel scan. Picard iteration therefore replaces
the sequential chain of time steps by a small number of global trajectory
updates.

However, time parallelism alone does not guarantee an efficient integrator.
Consider the linear autonomous system $\frac{\mathrm{d}}{\mathrm{d}t}X(t)=AX(t)$, $X(0)=X_0$;
 Picard iteration yields the \textit{Taylor expansion} for the matrix exponential, reconstructing the flow through successive polynomial approximations. When the linear
dynamics are mild, only a few iterations may suffice; in general, however, power-series methods scale poorly. Moreover, they are generally numerically unstable.

\subsection{Picard--Duhamel Integration}

We consider systems of the form
\begin{equation}
    \frac{\mathrm{d}}{\mathrm{d}t}X(t)
    =
    A(t)X(t) +\varepsilon F(X(t)),
    \qquad X(0)=X_0,
    \label{eq:weak}
\end{equation}
where $A(t)$ describes the linear dynamics and $\varepsilon F$ the
nonlinear contribution. Let $\Phi_s^t$ denote the linear propagator defined by the linear component:
\begin{equation}
    \frac{\partial}{\partial t}\Phi_s^t
    =
    A(t)\Phi_s^t,
    \qquad
    \Phi_s^s=I.
    \label{eq:linear_propagator}
\end{equation}
Thus, $\Phi_s^t v$ propagates a state $v$ from time $s$ to $t$ according
to the homogeneous linear dynamics. For numerical evaluation, we approximate the propagator over the temporal
grid by composing local matrix exponentials.
The Duhamel principle yields the analytical solution
\begin{equation}
    X(t)
    =
    \Phi_0^tX_0
    +
    \varepsilon\int_0^t
    \Phi_s^tF(X(s))\,\mathrm{d}s.
    \label{eq:duhamel}
\end{equation}
The linear evolution is therefore incorporated directly through
$\Phi_s^t$, rather than reconstructed through successive Picard iterations. Applying Picard iteration to \eqref{eq:duhamel} gives our
Picard--Duhamel scheme,
\begin{equation}
    X^{(k+1)}(t)
    =
    \Phi_0^tX_0
    +
    \varepsilon\int_0^t
    \Phi_s^tF(X^{(k)}(s))\,\mathrm{d}s,
    \qquad
    X^{(0)}(t)=\Phi_0^tX_0.
    \label{eq:duhamel-scheme}
\end{equation}
The initialization already captures the complete linear dynamics and is
exact when $\varepsilon=0$. Subsequent iterations therefore update only
the nonlinear correction while retaining the time-parallel structure of
classical Picard iteration.

\subsection{Discretization and Parallel Evaluation}

In practice, we evaluate the Picard--Duhamel scheme \eqref{eq:duhamel-scheme} using a numerical integration method. Let $0=t_0<t_1<\cdots<t_N=T$, with $h = \frac{T}{N}$ and $t_i= i h$. We use the trapezoidal method
\begin{align}
\label{eq:trapezoidal_method_def}
	\bar{X}_{j}^{(k+1)}:=\Phi_0^{t_j} X_0 + \frac{\varepsilon h}{2}\left(\Phi_0^{t_j} g^{(k)}_{0} +  2\sum_{i=1}^{j-1} \Phi_{t_i}^{t_j} g^{(k)}_{i} + g^{(k)}_{j}\right),
\end{align}
where $g_j^{(k)} = F(\bar{X}_{j}^{(k)})$.
A crucial property of this numerical integration scheme is that it can be evaluated in parallel over time. All nonlinear terms
$F(X_j^{(k)})$ depend only on the previous trajectory and can therefore
be evaluated simultaneously. Furthermore, the following recursion allows us to evaluate the cumulative integration as a parallel scan. Letting $P_j = \Phi_{t_j}^{t_{j+1}}$, we have
$\bar{X}_{j+1}^{(k+1)} = P_j \bar{X}_{j}^{(k+1)} + b_{j}$ with $ b_{j}=  \frac{\varepsilon h}2 (P_j g_{j}^{(k)} + g_{j+1}^{(k)})$. Two such maps compose according to
$ (P_b,b_b)\circ(P_a,b_a) = \left(P_bP_a,\;P_bb_a+b_b\right)$. Since composition of affine maps is associative, 

we can evaluate the iteration with a parallel scan, constructed through a tree of partial sums with $O(\log N)$ sequential depth and $O(N)$ total work. See Appendix~\ref{app:parallel_scans} for more details.

\subsection{Convergence and Error Bounds}

In the following theorem we establish convergence of the continuous Duhamel--Picard iteration \eqref{eq:duhamel-scheme}, and show that successive iterates approximate the solution with a factorially decaying error.

\begin{theorem}\label{thm:convergence}
	Consider \eqref{eq:weak} and let $F$ be any Lipschitz-continuous function with constant $L$ and suppose $\|A(t)\|_{\text{op}} \leq \alpha\, \forall t$. Then the Duhamel-Picard scheme \eqref{eq:duhamel-scheme} converges to $X$, the solution of \eqref{eq:weak}. Moreover, there exists a constant $C$ depending only on $X_0$, $F(X_0)$ and $L$ such that the following error bound holds:
	$$\|X^{(k)}(t) - X(t)\|\le C  \frac{L^k(\varepsilon t)^{k+1}}{(k+1)!}e^{(\alpha +\varepsilon L)t}$$
\end{theorem}

We next consider the numerical integration scheme that evaluates the integrals in the Duhamel--Picard iteration. The total error is decomposed into the error of the continuous Duhamel--Picard iteration and the quadrature error introduced at each time step. Combining Theorem~\ref{thm:convergence}
with the quadrature error for trapezoidal rules yields the following convergence estimates.

\begin{theorem}\label{thm:convergence_quadrature}
	Consider \eqref{eq:weak} under the same conditions of Theorem~\ref{thm:convergence} and let $C_2:=\sup_{\substack{k\ge0\\0\le s\le t\le T}}
	\left\|
	\frac{d^2}{ds^2}
	\left(\Phi_s^tF(X^{(k)}(s))\right)
	\right\|$. Then the error of the discrete Duhamel--Picard is bounded as follows.
	$$\|\bar{X}_{j}^{(k)} - X(t_j)\| \le\underbrace{C  \frac{L^k(\varepsilon t_j)^{k+1}}{(k+1)!}e^{(\alpha +\varepsilon L)t_j}	}_{\text{Duhamel-Picard iteration error}} + \underbrace{ \frac{\varepsilon h^3 C_2}{12} e^{\alpha t_{j-1}} \frac{(1 + \varepsilon L h)^{j+k - 1} - 1}{\varepsilon L h}}_{\text{Quadrature error}}$$
\end{theorem}

The factor $h^3$ is the local truncation-error factor of the trapezoidal method. Furthermore, the quadrature error increases with $j$ at least at a linear rate, which leads to the expected quadrature error of order $h^2$ of the trapezoidal method. The proof of Theorems~\ref{thm:convergence} and \ref{thm:convergence_quadrature} is covered in Appendix~\ref{app:proofs}.

\subsection{Training and inference with \method{}}\label{sec:training}

We now describe how the scheme is used inside a neuro-spectral model. The state $X$ is the vector of spectral coefficients $\hu$ (or $(\hu,\hat{\mathrm v})$ for second-order problems), the linear part $A$ is the analytic operator of~\eqref{eq:spectral-decomp} that NeuSA initialises around, and $\varepsilon F$ is the learned correction $\varepsilon \mathbf{F}_\theta$. In every problem we consider, $A$ is diagonal or block-diagonal in the basis, so $\Phi_s^t$ is a closed-form exponential and is cheap to evaluate.

\textbf{Linear part.}
Theorem~\ref{thm:convergence_quadrature} indicates that the speed of convergence of the Picard-Duhamel scheme may be substantially improved by reducing $\varepsilon$, the magnitude of the nonlinear part of the dynamics. Any linear term of the PDE whose propagator is available in closed form can then be moved from the network into the linear component; by Theorem~\ref{thm:convergence} this shrinks the constant $\varepsilon k$ that controls the number of sweeps, and it also lowers the effective stiffness the network has to represent.

\textbf{Warm-start training.}
Training minimises NeuSA's PDE residual along the trajectory, with the RK4 rollout replaced by the iteration~\eqref{eq:duhamel-scheme}. Consecutive optimiser steps change $\theta$ only slightly, so the trajectory produced at the previous step is already close to the fixed point of the current one. We therefore carry the trajectory across steps: instead of starting the iteration \ref{eq:duhamel-scheme} from the beginning at each step, the previous trajectory is taken as $X^{(0)}$ (with no gradient flowing into it):
\begin{equation}
    X^{(k+1)}(t)
    =
    \Phi_0^tX_0
    +
    \varepsilon\int_0^t
    \Phi_s^tF(X^{(k)}(s))\,\mathrm{d}s,
    \qquad
    X^{(0)}(t)=X^{(0)}_{n-1}(t).
    \label{eq:duhamel-scheme-warm}
\end{equation}
$K$ new sweeps are computed from it, and the loss is evaluated on $X^{(K)}$ and differentiated through those $K$ sweeps only (Algorithm~\ref{alg:papi}). Two sweeps per step suffice on all but the most strongly nonlinear problem we consider. Since the residual of the iterate is smallest at early times and grows with $t$ like $(\varepsilon k t)^{K+1}$ (Theorem~\ref{thm:convergence}), we weight the residual in time by $w(t)=(1+t)^{-2}$, which anchors the loss on the part of the trajectory that is already converged while the later part catches up over subsequent steps. 
\begin{algorithm}[t]
\caption{One optimiser step of \method{} (warm-started Picard--Duhamel training).}
\label{alg:papi}
\begin{algorithmic}[1]
\REQUIRE parameters $\theta$, carried trajectory $\bar X=(\bar X_0,\ldots,\bar X_N)$, propagators $P_j=\Phi_j^{j+1}$, sweeps $K$, weights $w_j$
\STATE $X^{(0)}\leftarrow \operatorname{stop\_gradient}(\bar X)$
\FOR{$k=1,\ldots,K$}
  \STATE $g_j \leftarrow \varepsilon F_\theta(X^{(k-1)}_j)$ for all $j$ \hfill (parallel over $j$)
  \STATE $b_j \leftarrow$ quadrature of $\int_{t_j}^{t_{j+1}}\Phi_s^{t_{j+1}}g(s)\,\mathrm ds$ from $g_j,g_{j+1}$
  \STATE $X^{(k)}\leftarrow \operatorname{scan}\big((P_j,b_j)_{j}\big)$ from $X_0$ \hfill (parallel prefix)
\ENDFOR
\STATE $\mathcal L(\theta)\leftarrow \sum_j w_j\,\|\text{residual}_j(\theta, X^{(K)})\|^2/\sum_j w_j$
\STATE $\theta\leftarrow\operatorname{Adam}(\theta,\nabla_\theta\mathcal L)$; \quad $\bar X\leftarrow X^{(K)}$
\end{algorithmic}
\end{algorithm}

\textbf{Exponential quadrature rules.}
We use an exponential variant of the trapezoidal method \eqref{eq:trapezoidal_method_def} that integrates $\Phi_s^{t}$ exactly against a piecewise-linear interpolant of $g$. For brevity, we introduce it in Appendix~\ref{app:exponential_trapezoidal}. 

\textbf{Inference.}
The trained model is deployed with a cold iteration from the linear flow, $X^{(0)}=\Phi_0^tX_0$, for a fixed number of sweeps
; eight sweeps suffice on all but the stiffest problem we consider, where 32 are used. Each sweep is one batched network evaluation and one scan of depth $O(\log N)$, so inference has sequential depth $O(K\log N)$ rather than $O(N)$.

\section{Experiments}\label{sec:experiments}

We evaluate \method{} on the 2D wave equation in a three-layer medium (2D Layers) and in the Marmousi model, the 3D wave equation in a two-layer medium (3D Layers), the coupled 2D Burgers’ equation, the 1D sine-Gordon equation, and the Korteweg–de Vries (KdV) equation, whose dispersive linear term makes the spectral ODE stiff. We also test \method{} on an inverse Burgers' problem to identify PDE coefficients from data efficiently. We describe the PDEs and experimental setup in Appendix~\ref{app:details}.

\textbf{Baselines.}
We follow the benchmarks of \citet{bizzi-neusa}, comparing against PINN~\citep{raissi2019physics}, QRes~\citep{bu2021qres}, FLS~\citep{wong2024fls}, and NeuSA. Appendix~\ref{app:details} provides the baseline configurations and adaptations for the additional KdV and inverse Burgers' problems.

\textbf{Implementation details.}
\method{} is implemented in JAX~\citep{jax2018github}. All experiments were run on an NVIDIA RTX~5090 GPU. For comparison with NeuSA, \method{} replaces the integrator but retains each problem's other parameters, including the basis, mode count, dimension-wise MLP with two hidden layers and residual loss, unless stated otherwise. The linear part $A$ matches NeuSA's initial linear operator; Appendix~\ref{app:details} lists exceptions, including the sine-Gordon mass term.
Training follows Section~\ref{sec:training} with $K=2$ warm-started Picard--Duhamel sweeps per optimiser step (six on Marmousi), residual weight $(1+t)^{-2}$, and Adam with the reported spectral-model learning rates, decay schedules and step counts unless stated otherwise. Inference uses 8 cold sweeps (32 on KdV; Appendix~\ref{app:sweeps}). Training time is the wall-clock time of the optimisation loop, excluding JAX compilation (5--20\,s per run). Forward results are means over five seeds per configuration; Appendix~\ref{app:tables} gives standard deviations. We report relative $L^2$ (rMSE) and $L^1$ (rMAE) errors against the reference solution on the evaluation grid, plus relative coefficient errors for the inverse problem.

\subsection{Comparisons}\label{sec:neusa-benchmarks}

Table~\ref{tab:main} reports rMSE and training time for the six forward problems; Figure~\ref{fig:wave} shows the Marmousi solution, and Appendix~\ref{app:tables} gives rMAE, standard deviations, the error as a function of time and a time-against-error plot for every benchmark.

\begin{table}[!ht]
\centering
\caption{\textbf{Quantitative comparison.} Mean rMSE and training time over five seeds for each model. TT denotes training time in seconds. \method{} has the shortest training time on all six problems and the lowest rMSE on five. Appendix~\ref{app:tables} reports standard deviations and rMAE.}
\label{tab:main}
\small
\setlength{\tabcolsep}{3pt}
\resizebox{\textwidth}{!}{%
\begin{tabular}{l rr rr rr rr rr rr}
\toprule
& \multicolumn{2}{c}{2D Layers} & \multicolumn{2}{c}{Marmousi} & \multicolumn{2}{c}{3D Layers} & \multicolumn{2}{c}{2D Burgers} & \multicolumn{2}{c}{Sine-Gordon} & \multicolumn{2}{c}{KdV}\\
\cmidrule(lr){2-3}\cmidrule(lr){4-5}\cmidrule(lr){6-7}\cmidrule(lr){8-9}\cmidrule(lr){10-11}\cmidrule(lr){12-13}
Model & rMSE & TT & rMSE & TT & rMSE & TT & rMSE & TT & rMSE & TT & rMSE & TT\\
\midrule
PINN  & 0.549 & 420 & 0.618 & 427 & 0.062 & 4000 & 0.183 & 605  & 0.130 & 569 & 6.1e-3 & 908\\
QRes  & 0.150 & 765 & 0.401 & 707 & 0.026 & 5496 & 0.073 & 1178 & 0.013 & 781 & 9.5e-3 & 1240\\
FLS   & 0.492 & 408 & 0.712 & 408 & 0.071 & 4050 & 0.174 & 578  & 0.129 & 594 & 1.5e-2 & 935\\
NeuSA & 0.069 & 508 & 0.266 & 604 & \textbf{0.008} & 577 & 0.053 & 119 & 9.0e-4 & 234 & 3.6e-3 & 1322\\
\method{} (ours) & \textbf{0.068} & \textbf{12} & \textbf{0.154} & \textbf{66} & 0.018 & \textbf{251} & \textbf{0.028} & \textbf{30} & \textbf{5.6e-4} & \textbf{0.3} & \textbf{1.2e-3} & \textbf{3.2}\\
\bottomrule
\end{tabular}}
\end{table}

\begin{figure}[!ht]
    \centering
    \includegraphics[width=0.8\textwidth]{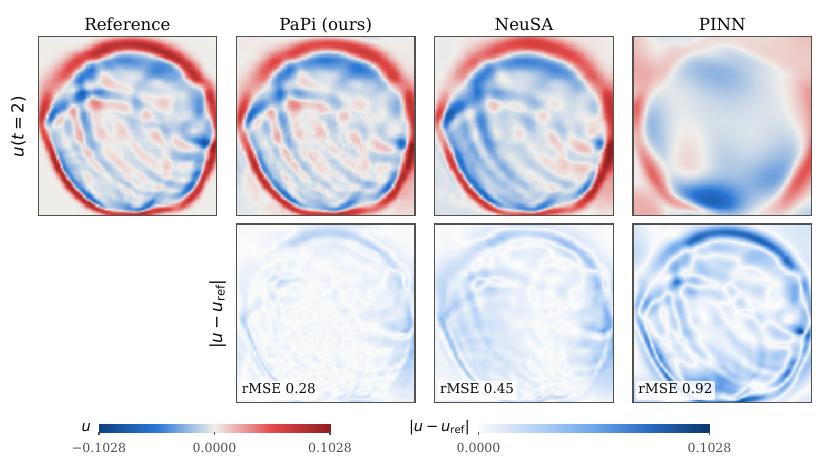}
    \caption{Wave equation in the Marmousi medium at $t=2$ (seed 0): the displacement $u$ on the region of interest $[-2,2]^2$ (top) and the absolute error against the reference (bottom), with the relative $L^2$ error at this time in each error tile. \method{} resolves the reflected and refracted fronts that NeuSA blurs and the PINN misses.}
    \label{fig:wave}
\end{figure}

\textbf{Wave equation.}
On 2D Layers \method{} matches the accuracy of NeuSA (rMSE $0.068$ against $0.069$) while training in 12\,s instead of 508\,s. On Marmousi, the most heterogeneous medium, it is more accurate ($0.15$ against $0.27$)\footnote{Our rerun of the authors' code at the paper's configuration gives $0.27\pm0.09$, above the $0.171$ reported in NeuSA; see Appendix~\ref{app:details}.} and almost $10.0\times$ faster. The speedup is smaller here for two reasons: the learned term, which represents the whole heterogeneity $(c^2-1)\Delta u$ with speeds between $0.57$ and $2.18$, is not weak, so six sweeps per step are needed; and the RK4 rollout is already compute-bound on a $201^2\times2$ state, so the parallel scan removes depth but not arithmetic. On 3D Layers \method{} trains $2.3\times$ faster than NeuSA but is less accurate ($0.018$ against $0.0076$), which we attribute to the coarse training grid (56 of the 111 time nodes) and to the large learned term ($\epsilon=10$ in NeuSA's setting); the PINN baselines are worse by a further factor of $1.5$--$4$ and take $7$--$20$ times longer.

\textbf{Sine-Gordon.}
This is the problem on which the integrator matters most. With the mass term carried by the linear part, \method{} reaches an rMSE of $5.6\times10^{-4}$, below NeuSA's $9.0\times10^{-4}$, in $0.3$\,s of training against $234$\,s, i.e.\ over two orders of magnitude faster, and three orders of magnitude faster than the PINN baselines, which are also one to two orders of magnitude less accurate. The optimisation loop takes less than a second because each optimiser step is two parallel sweeps over a $402$-dimensional state on 201 nodes, and the 1{,}000 steps are compiled into one program.

\textbf{2D Burgers.}
NeuSA is trained for 200 steps, as in its paper; \method{} is trained for 1{,}000 steps, which still takes a quarter of the time (30\,s against 119\,s) and halves the error (rMSE $0.028$ against $0.053$). The best PINN baseline, QRes, reaches $0.073$ in 1{,}178\,s.

\subsection{Korteweg--de Vries equation}\label{sec:kdv}

The KdV equation
\begingroup
\setlength{\abovedisplayskip}{4pt}
\setlength{\belowdisplayskip}{4pt}
\setlength{\abovedisplayshortskip}{0pt}
\setlength{\belowdisplayshortskip}{4pt}
\begin{equation}\label{eq:kdv}
    \partial_tu + 6u\,\partial_xu + \partial_{xxx}u = 0,\qquad u(0,x)=\frac{c}{2}\,\mathrm{sech}^2\!\left(\frac{\sqrt c}{2}(x-x_0)\right),
\end{equation}
\endgroup
on $[-10,10]\times[0,1]$ with periodic boundary conditions, $c=4$ and $x_0=-6$, has the one-soliton solution $u(t,x)=u(0,x-ct)$, which we use as reference. In a Fourier basis the linear term is the dispersion $-\partial_{xxx}$, whose eigenvalues $\pm ik^3$ grow with the cube of the wavenumber; the spectral ODE is therefore stiff, leading to NeuSA's collapse~\citep[App.~E]{bizzi-neusa}, as explicit RK4 is stable only for $\Delta t\,|k|_{\max}^3\lesssim 2.8$.
\method{} treats the dispersion exactly through $\Phi_s^t=e^{ik^3(t-s)}$, so the grid only has to resolve the smooth soliton and the learned correction. We use 81 modes for both methods so that they can be compared, $\epsilon=0.1$, and train \method{} for 10{,}000 steps with a slower learning-rate decay; the converged Picard solution requires 32 sweeps at inference (Appendix~\ref{app:sweeps}).

\method{} reaches an rMSE of $1.2\times10^{-3}$ in $3.2$\,s. NeuSA, which must take 1{,}001 sequential RK4 steps per optimiser step, takes 1{,}322\,s for 1{,}000 steps and an rMSE of $3.6\times10^{-3}$; the PINN baselines lie between $6\times10^{-3}$ and $1.5\times10^{-2}$ after 900--1{,}200\,s. At 101 modes, where RK4 diverges, \method{} trains in $3.4$\,s to an rMSE of $1.4\times10^{-3}$.

\subsection{Inverse problem: coefficient identification in Burgers' equation}\label{sec:inverse}

Since the Picard--Duhamel rollout is differentiable with respect to the linear operator as well as to the network, \method{} applies without change to inverse problems in which coefficients of the PDE are unknown. As a test of its ability to identify such coefficients from sparse data we take the parameter-identification problem for Burgers' equation as in \citet{raissi2019physics}, given by:
\begin{equation}\label{eq:inverse}
    \partial_tu + \lambda_1 u\,\partial_xu = \lambda_2\,\partial_{xx}u,\qquad u(0,x)=-\sin(\pi x),
\end{equation}
on $[-1,1]\times[0,1]$ with periodic boundary conditions, where $\lambda_1$ and $\lambda_2$ are to be learned and the true values are $(\lambda_1,\lambda_2)=(1,\,0.1)$. The model sees two snapshots of the solution: the initial condition, which it satisfies by construction, and the final state $u(1,x)$ on the 199-node spatial grid, either clean or corrupted with Gaussian noise of standard deviation equal to 1\% of that of the data. Both coefficients are trained jointly with the network by Adam. Since the outcome depends on where the coefficients start, the experiment is repeated for ten random initial guesses $(\lambda_1^0,\lambda_2^0)$ drawn uniformly from $[0.1,2]\times[0.01,0.2]$, one seed each, and the results are averaged.

As with KdV, the problem is stiff: $\lambda_2\,k_{\max}^2\approx10^4$ at the true viscosity, so on the data grid an explicit RK4 rollout is unstable before training even begins (it would need a time step ten to twenty times smaller), which is why sequential NeuSA does not appear in Table~\ref{tab:inverse}; the exact propagator is unaffected. 
The PINN, QRes and FLS baselines treat $(\lambda_1,\lambda_2)$ as trainable scalars in the residual and are trained for 5{,}000 steps. We report \method{} at a budget of 500 steps and, since it is cheap, at 5{,}000 steps in double precision.

Table~\ref{tab:inverse} reports relative solution errors on $[0,1]$, relative coefficient errors and training time. Without temporal causality, the MLP baselines struggle to reconstruct the solution and recover coefficients from two widely separated snapshots. QRes performs best among them, with errors of 20\% for $\lambda_1$ and 4\% for $\lambda_2$. At 500 steps, \method{} reaches 5\% and 1\%, respectively, in under a second, over two orders of magnitude faster. At 5{,}000 steps with clean data, both coefficient errors reach $10^{-4}$ in about a minute; with 1\% noise, they settle at $4\times10^{-3}$ for $\lambda_1$ and $2\times10^{-3}$ for $\lambda_2$.

Evolving the identified equation with the reference solver on $(1,2]$, beyond the observed interval, gives relative solution errors of $1.0\times10^{-2}$ after 500 training steps and $5\times10^{-5}$ (clean) or $2\times10^{-3}$ (noisy) after 5{,}000 steps. As in NeuSA, this extrapolation comes from the architecture's causal structure.

\begin{table}[t]
\centering
\caption{Inverse Burgers' problem: mean $\pm$ std over the ten starting pairs of the relative $L^2$ error on $[0,1]$, the relative errors of the identified coefficients and the training time per run. The observation is the final frame only; noise is relative to the standard deviation of the data.}
\label{tab:inverse}
\small
\begin{tabular}{l l r r r r}
\toprule
Noise & Model & rMSE & $|\lambda_1-1|$ & $|\lambda_2-0.1|/0.1$ & TT (s)\\
\midrule
0\%    & PINN (5{,}000 steps) & 9.9e-2 $\pm$ 4.1e-2 & 4.9e-1 $\pm$ 2.0e-1 & 8.2e-2 $\pm$ 3.5e-2 & 419\\
     & QRes (5{,}000 steps) & 4.0e-2 $\pm$ 2.1e-2 & 2.0e-1 $\pm$ 1.0e-1 & 3.7e-2 $\pm$ 1.9e-2 & 561\\
     & FLS  (5{,}000 steps) & 9.9e-2 $\pm$ 4.1e-2 & 4.9e-1 $\pm$ 2.0e-1 & 8.1e-2 $\pm$ 3.5e-2 & 411\\
     & \method{} (500 steps) & 1.0e-2 $\pm$ 8.4e-3 & 4.7e-2 $\pm$ 4.1e-2 & 9.4e-3 $\pm$ 8.2e-3 & \textbf{0.6}\\
     & \method{} (5{,}000 steps, float64) & \textbf{3.8e-5} $\pm$ 1.2e-5 & \textbf{1.1e-4} $\pm$ 5.3e-5 & \textbf{3.6e-5} $\pm$ 2.2e-5 & 66\\
\midrule
1\% & PINN (5{,}000 steps) & 9.8e-2 $\pm$ 4.1e-2 & 4.9e-1 $\pm$ 2.0e-1 & 8.0e-2 $\pm$ 3.5e-2 & 432\\
     & QRes (5{,}000 steps) & 4.0e-2 $\pm$ 2.2e-2 & 1.9e-1 $\pm$ 1.0e-1 & 3.6e-2 $\pm$ 1.9e-2 & 565\\
     & FLS  (5{,}000 steps) & 9.8e-2 $\pm$ 4.2e-2 & 4.9e-1 $\pm$ 2.0e-1 & 8.0e-2 $\pm$ 3.5e-2 & 406\\
     & \method{} (500 steps) & 9.8e-3 $\pm$ 8.5e-3 & 4.7e-2 $\pm$ 4.1e-2 & 9.4e-3 $\pm$ 8.3e-3 & \textbf{0.7}\\
     & \method{} (5{,}000 steps, float64) & \textbf{1.4e-3} $\pm$ 5.0e-4 & \textbf{4.2e-3} $\pm$ 2.7e-3 & \textbf{2.0e-3} $\pm$ 1.4e-3 & 65\\
\bottomrule
\end{tabular}
\end{table}

\subsection{Discussion}\label{sec:discussion}

Across the forward problems \method{} matches or improves NeuSA's accuracy on five of six benchmarks while reducing training time by a factor that ranges from $2.3$ (3D Layers) to about $800$ (sine-Gordon), and it reaches errors one to two orders of magnitude below the MLP-based PINNs at a small fraction of their cost. The size of the speedup follows the sequential depth that is removed: for the 1D problems, NeuSA's cost is dominated by the $4\times(\text{time steps})$ dependent kernel launches of the RK4 rollout, and replacing them with a handful of parallel sweeps brings training below a few seconds; for the 2D and 3D problems, each step is already compute-bound on the spectral state, and the gain is the ratio between the arithmetic of an RK4 rollout and that of $K$ sweeps, which is between $2$ and $10$. Part of the difference between NeuSA and \method{} on the 1D problems is due to the framework rather than to the integrator; Appendix~\ref{app:jax} quantifies this with an RK4 rollout implemented in JAX, and shows that the remaining factor of $50$--$200$ on the 1D problems, and of $2$--$10$ on the 2D and 3D problems, is due to \method{}. Finally, KdV and the inverse Burgers' problem illustrate the second advantage of the method: since the linear part is propagated exactly, the stiffness that constrains explicit time stepping in NeuSA, and forces a trade-off between spatial resolution and time step~\citep[App.~D.1]{bizzi-neusa}, does not enter \method{}, whose time grid only needs to resolve the nonlinear correction.

\section{Conclusions}

We introduce \method{}, a time-parallel integration scheme designed to accelerate the training and inference of Neuro-Spectral Architectures. By combining exact linear propagation with Picard--Duhamel iterations and parallel scans, our approach replaces sequential time stepping with parallel trajectory updates, addressing a major computational bottleneck of the original NeuSA formulation without modifying its underlying neural representation. We further established convergence of the proposed iteration and derived explicit error bounds for its numerical approximation.

Our experiments demonstrate substantial computational gains across a range of time-dependent PDEs, including forward and inverse problems. In several cases, \method{} reduces training times by orders of magnitude while maintaining competitive predictive accuracy, with sub-second training achieved on low-dimensional benchmarks. The computational benefits nevertheless depend on the number of iterations required to resolve the nonlinear dynamics. Improving the efficiency of these iterations, particularly for strongly nonlinear problems, is a natural direction for future work.

\section{AI Use Statement}
We used generative AI tools to aid and polish writing, in particular editing the paper for readability (grammar, clarity, and phrasing of text written by the authors), and to create and edit auxiliary code.

\section{Acknowledgments}

Arthur Bizzi is supported by UK Research and Innovation (UKRI) through the Engineering and Physical Sciences Research Council (EPSRC).
Tiago Novello is thankful for financial support from Google. João M Pereira is thankful for a start-up grant from the University of Georgia. Vitor Balestro gratefully acknowledges Centro Pi/IMPA for providing computational resources used in this work.

\bibliographystyle{iclr2027_conference}

\appendix

\section{Proof of Convergence and Error Bounds}\label{app:proofs}

We first state and prove some necessary lemmas.
\begin{lemma}\label{lemma:Gronwall}
Let $\Phi_{s}^t$ be defined as in \eqref{eq:linear_propagator}, with $s\le t$ and $\|A(t)\|_{\text{op}} \leq \alpha  \,\forall t$. Then $\|\Phi_{s}^t\|_{\text{op}} \leq e^{\alpha(t-s)}$ and \linebreak $\|\Phi_{s}^t-I\|_{\text{op}} \le e^{\alpha(t-s)} -1$, where $I$ is the identity matrix.
\end{lemma}

\begin{proof}
Let $z$ be any vector, and define $Z(t)=\Phi_{s}^{t} z$. $Z(t)$ admits the integral form:
$$Z(t)=Z(s) + \int_{s}^t A(\tau) Z(\tau) d\tau.$$
Taking the norm, we obtain
\begin{align}
\nonumber\|Z(t)\|&\le \|Z(s)\| + \int_{s}^t \| A(\tau) Z(\tau)\|d\tau,\\
&\le \|Z(s)\| + \int_{s}^t \alpha \|Z(\tau)\|d\tau. \label{eq:pre_gronwall}
\end{align}
Applying Grönwall's inequality \cite[Theorem 1.2.2]{pachpatte1998inequalities} to \eqref{eq:pre_gronwall}, we obtain $\forall s,t$
\begin{align*}
\|\Phi_{s}^{t} z \| = \|Z(t)\|
&\le e^{\alpha(t-s)} \|Z(s)\| \\ &= e^{\alpha(t-s)} \|z\|
\end{align*}
Using this inequality, we obtain
\begin{align*}
\|(\Phi_{s}^{t} - I) z \| &= \|Z(t) - Z(s)\|\\
&\le \int_{s}^t \alpha \|Z(\tau)\|d\tau\\
&\le \alpha \int_{s}^t e^{\alpha(\tau-s)} \|z\|d\tau\\
&= (e^{\alpha(t-s)} - 1) \|z\|
\end{align*}
Since this holds for any vector $z$, the proof follows.
\end{proof}

\begin{lemma}\label{lemma:m_recursion_2}
	Let $q\ge 1$, $r,s\ge 0$, and suppose $\{m^{(k)}_{j}\}$ follows the recurrence
	\begin{equation}\label{eq:recurrence}
		m_{j+1}^{(k+1)} \le q m_j^{(k+1)} + r\bigl(m_{j+1}^{(k)}+q m_j^{(k)}\bigr) + s
	\end{equation}
	for all $j,k\ge 0$, with $m^{(k)}_{0}=m^{(0)}_{j}=0$. Then, for every $j\ge 1$ and $k\ge 1$
	\begin{equation}
		\label{eq:m-bound}
		m_j^{(k)}
		\le
		s q^{j-1} \frac{(1 + 2r)^{j+k - 1} - 1}{2r},
	\end{equation}
	with $\frac{(1 + 2r)^{j+k - 1} - 1}{2r}=j+k-1$ if $r=0$.
\end{lemma}

\begin{proof}
	Define the rescaled sequence
	\[
	a_j^{(k)}:=q^{1-j}m_j^{(k)}.
	\]
	Dividing \eqref{eq:recurrence} by $q^{j}$ gives
	\[
	a_{j+1}^{(k+1)}
	\le
	a_j^{(k+1)}
	+r\bigl(a_{j+1}^{(k)}+a_j^{(k)}\bigr)
	+s q^{-j}.
	\]
	Since $q\ge 1$, we have $q^{-j}\le 1$, and therefore
	\begin{equation}
		\label{eq:scaled-one-step}
		a_{j+1}^{(k+1)}
		\le
		a_j^{(k+1)}
		+r\bigl(a_{j+1}^{(k)}+a_j^{(k)}\bigr)
		+s.
	\end{equation}
	Summing \eqref{eq:scaled-one-step} over $j$ index from $0$ to
	$j-1$ and using $a_0^{(k+1)}=a_0^{(k)}=0$, we obtain
	\begin{align}
		a_j^{(k+1)}
		&\le
		r\sum_{i=0}^{j-1}
		\bigl(a_{i+1}^{(k)}+a_i^{(k)}\bigr)+js \notag\\
		&=
		r\left(a_j^{(k)}+2\sum_{i=1}^{j-1}a_i^{(k)}\right)+js,
		\label{eq:summed-recurrence}
	\end{align}
	We now prove by induction on $k$ that
	\begin{equation}
		\label{eq:scaled-finite-bound}
		a_j^{(k+1)}
		\le
		s\sum_{\ell=0}^{k}
		(2r)^\ell\binom{j+\ell}{\ell+1}
	\end{equation}
	for every $j\ge 1$.
	
	For $k=0$, the condition $a_j^{(0)}=0$ and
	\eqref{eq:summed-recurrence} imply $a_j^{(1)}\le js
	=s\binom{j}{1}$, and for $k>0$, substituting \eqref{eq:scaled-finite-bound}
	into \eqref{eq:summed-recurrence} yields
	\begin{align*}
		a_j^{(k+1)}
		\le{}& js
		+rs\sum_{\ell=0}^{k-1}(2r)^\ell
		\left[
		\binom{j+\ell}{\ell+1}
		+2\sum_{i=1}^{j-1}\binom{i+\ell}{\ell+1}
		\right].
	\end{align*}
	The hockey-stick identity gives
	\[
	\sum_{i=1}^{j-1}\binom{i+\ell}{\ell+1}
	=\binom{j+\ell}{\ell+2}.
	\]
	Moreover, by nonnegativity and Pascal's identity,
	\begin{align*}
		\binom{j+\ell}{\ell+1}
		+2\binom{j+\ell}{\ell+2}
		&\le
		2\left[
		\binom{j+\ell}{\ell+1}
		+\binom{j+\ell}{\ell+2}
		\right]\\
		&=2\binom{j+\ell+1}{\ell+2}.
	\end{align*}
	Consequently,
	\begin{align*}
		a_j^{(k+1)}
		&\le
		js+s\sum_{\ell=0}^{k-1}
		(2r)^{\ell+1}\binom{j+\ell+1}{\ell+2}\\
		&=
		js+s\sum_{\ell=1}^{k}
		(2r)^{\ell}\binom{j+\ell}{\ell+1}\\
		&=
		s\sum_{\ell=0}^{k}
		(2r)^\ell\binom{j+\ell}{\ell+1},
	\end{align*}
	This completes the induction and since $m_j^{(k+1)}=q^{j-1} a_j^{(k+1)}$,
	$$m_j^{(k)} \le s q^{j-1}\sum_{\ell=0}^{k-1}
	(2r)^\ell\binom{j+\ell}{\ell+1}.$$
	Equation \eqref{eq:m-bound} follows
	\begin{align*}
		m_j^{(k)} &\le s q^{j-1}\sum_{\ell=0}^{k-1}
		(2r)^\ell\binom{j+\ell}{\ell+1}\\
		&\le s q^{j-1}\sum_{\ell=0}^{k-1}
		(2r)^\ell\binom{j+k-1}{\ell+1}\\
		&\le s q^{j-1}\sum_{\ell=0}^{j+k-2}
		(2r)^\ell\binom{j+k-1}{\ell+1}\\
		&= s q^{j-1} \frac{(1 + 2r)^{j+k - 1} - 1}{2r}
	\end{align*}.
\end{proof}

\begin{proof}[Proof of Theorem~\ref{thm:convergence}]
	The proof is inspired by the classical convergence argument for Picard iteration due to Lindelöf \cite{lindelof1894application}. %
	Let $Y^{(k)} = X^{(k+1)}-X^{(k)}$. For $k=0$, we have the following estimate:
	\begin{align}
		\nonumber\|Y^{(0)}\| &= \|X^{(1)}-X^{(0)}\|\leq \varepsilon \int_0^t \left\|\Phi_s^t (F(X^{(0)}(s))\right\|\, \d s,\\
		\nonumber&\le \varepsilon \int_0^t e^{\alpha(t-s)}\left\|F(X^{(0)}(s))\right\|,\\
		\nonumber&\le \varepsilon \int_0^t e^{\alpha(t-s)}\left(\|F(X^{(0)}(s))-F(X_0)\| + \|F(X_0)\|\right),\\
		\nonumber&\le \varepsilon \int_0^t e^{\alpha(t-s)}\left(L \|X^{(0)}(s)-X_0\| + \|F(X_0)\|\right),\\
		\nonumber&\le \varepsilon \int_0^t e^{\alpha(t-s)}\left(L (e^{\alpha s} - 1) \|X_0\| + \|F(X_0)\|\right),\\
		\nonumber&= \frac{\varepsilon}{\alpha} \left(((t\alpha - 1) e^{\alpha t} + 1) L \|X_0\| + (e^{\alpha t} - 1)\|F(X_0)\|\right),\\
		\label{eq:ind_base_case} &\le \varepsilon (L \|X_0\| + \|F(X_0)\|) t e^{\alpha t},
	\end{align}
	where we use Lemma~\ref{lemma:Gronwall} in lines 1 and 5. For $k>0$, we obtain the following recurrence:
	\begin{align*}
		\|Y^{(k+1)}(t)\| &\leq \varepsilon \int_0^t\left\| \Phi_s^t (F(X^{(k+1)}(s))  - F(X^{(k)}(s))\right\|\, \d s
		\\&\leq \varepsilon \int_0^t e^{\alpha(t-s)} \left\|(F(X^{(k+1)}(s))  - F(X^{(k)}(s))\right|\, \d s
		\\&\leq \varepsilon L \int_0^t e^{\alpha(t-s)} \left\|X^{(k+1)}(s) - X^{(k)}(s) \right\|\, \d s
		\\&\leq \varepsilon L\int_0^t e^{\alpha(t-s)} \left\|Y^{(k)}(s)\right\|\, \d s
	\end{align*}
	Letting $C=L\|X_0\| + \|F(X_0)\|$, we now show by induction over $k$ that:
	\begin{equation}\label{eq:picard_estimate}
		\|Y^{(k)}(t)\| \leq C \frac{L^{k}(\varepsilon  t)^{k+1}}{(k+1)!}e^{\alpha t}
	\end{equation}
	The case $k=0$ follows from $\eqref{eq:ind_base_case}$, while for $k+1$
	\begin{align*}
		\|Y^{(k+1)}(t)\| &\leq \varepsilon L\int_0^t e^{\alpha(t-s)} \left\|Y^{(k)}(s)\right\|\, \d s\\
		&\leq C \frac{L^{k+1}\varepsilon^{k+2}}{(k+1)!} \int_0^t e^{\alpha(t-s)}e^{\alpha s} s^{k+1} \, \d s\\
		&= C \frac{ L^{k+1}(\varepsilon t)^{k+2}}{(k+2)!}e^{\alpha t}
	\end{align*}
	We have $\sum_{k=0}^\infty |Y^{(k)}|<\infty$, which implies that $X^{(k)}$ is a Cauchy sequence, hence it converges, and we have
	\begin{align*}
		\|X^{(k)}(t) - X(t)\|&\le \sum_{i=k}^{\infty}\|X^{(i+1)}(t)-X^{(i)}(t)\|= \sum_{i=k}^{\infty}\|Y^{(i)}(t)\|\\
		&\le \frac{C}{L} e^{\alpha t} \sum_{i=k}^{\infty}  \frac{(\varepsilon L t)^{i+1}}{(i+1)!}.
	\end{align*}
	Finally, the error bound follows from $(k+i)!\ge k! i !$:
	\[\sum_{i=k}^{\infty}  \frac{(\varepsilon L t)^{i+1}}{(i+1)!} = \frac{(\varepsilon L t)^{k+1}}{(k+1)!} \sum_{i=0}^{\infty} \frac{(k+1)!}{(i+k+1)!}(\varepsilon L t)^{i} \le \frac{(\varepsilon L t)^{k+1}}{(k+1)!} \sum_{i=0}^{\infty} \frac{1}{i!}(\varepsilon L t)^{i} = \frac{(\varepsilon L t)^{k+1}}{(k+1)!}e^{\varepsilon L t}.\qedhere\]
\end{proof}

\begin{proof}[Proof of Theorem~\ref{thm:convergence_quadrature}]
	Let $m_{j}^{(k)} = \|\bar{X}_{j}^{(k)} - X^{(k)}(t_j)\|$. Applying the trapezoidal bounds and the triangle inequality we obtain
	\begin{align*}
	m_{j+1}^{(k+1)} &= \|\bar{X}_{j+1}^{(k+1)} - X^{(k+1)}(t_{j+1})\| \\
	&\le \underbrace{\|P_j \bar{X}_{j}^{(k+1)} - P_j X^{(k+1)}(t_j)\|}\limits_{(1)} + \underbrace{\|b_{j} -  \tilde b_{j}\|}\limits_{(2)} + \underbrace{\|\Delta X^{(k+1)}(t_j) - \tilde b_{j}\|}\limits_{(3)},
	\end{align*}
	where 
	\begin{align*}
	b_{j}&=\frac{\varepsilon h}{2} (P_j g_{j}^{(k)} + g_{j+1}^{(k)}),\\
	g_{j}^{(k)}&=F(\bar{X}_{j}^{(k)}),\\
	\tilde b_{j}&=\frac{\varepsilon h}{2} (P_j \tilde g_{j}^{(k)} + \tilde g_{j+1}^{(k)}),\\
	\tilde g_{j}^{(k)}&=F(X^{(k)}(t_{j})),\\
	\Delta X^{(k+1)}(t_j)&= X^{(k+1)}(t_{j+1}) - P_j X^{(k+1)}(t_j).
	\end{align*}
	Using $\|P_j\|_{op} \le e^{\alpha h}$, we have 
	$$\|P_j \bar{X}_{j}^{(k+1)} - P_j X^{(k+1)}(t_j)\|\le e^{\alpha h}\|\bar{X}_{j}^{(k+1)} - X^{(k+1)}(t_j)\| = e^{\alpha h}m_{j}^{(k+1)}$$
	Similarly, we have
	\begin{align*}
	\|b_{j} -  \tilde b_{j}\| &\le \frac{\varepsilon h}{2}\left( e^{\alpha h}\|g_{j}^{(k)} - \tilde g_{j}^{(k)}\| + \|g_{j+1}^{(k)} - \tilde g_{j+1}^{(k)}\|\right)\\
	&\le \frac{\varepsilon h}{2}\left( e^{\alpha h}\|\bar{X}_{j}^{(k)} -  X^{(k)}(t_{j})\| + \|\bar{X}_{j+1}^{(k)} -  X^{(k)}(t_{j+1})\|\right)\\
	&= \frac{\varepsilon L h}2 (e^{\alpha h} m_{j}^{(k)} + m_{j+1}^{(k)}),
	\end{align*}
	Finally, \eqref{eq:duhamel-scheme} and the truncation error of the trapezoidal method implies 
	\begin{align*}
		\|\Delta X^{(k+1)}(t_j) - \tilde b_{j}\|&=\left\|\varepsilon\int_{t_j}^{t_{j+1}}
		\Phi_s^{t_{j+1}}F(X^{(k)}(s))\,\mathrm{d}s - \tilde b_{j}\right\|\\
		&=\varepsilon\left\|\int_{t_j}^{t_{j+1}}
		\Phi_s^{t_{j+1}}F(X^{(k)}(s))\,\mathrm{d}s - \frac{\varepsilon h}{2} (P_j \tilde g_{j}^{(k)} + \tilde g_{j+1}^{(k)})\right\|\\	
		&\le \frac{\varepsilon h^3 C_2}{12}
	\end{align*}
	We obtain
	\begin{align*}
		m_{j+1}^{(k+1)} \le e^{\alpha h}m_{j}^{(k+1)} + \frac{\varepsilon L h}2 (e^{\alpha h} m_{j}^{(k)} + m_{j+1}^{(k)})  + \frac{\varepsilon h^3 C_2}{12}
	\end{align*}
	Applying Lemma~\ref{lemma:m_recursion_2} we obtain 
	$$m_{j}^{(k)} \le \frac{\varepsilon h^3 C_2}{12} e^{\alpha t_{j-1}} \frac{(1 + \varepsilon L h)^{j+k - 1} - 1}{\varepsilon L h}$$
	Finally, applying Theorem~\ref{thm:convergence},
	\begin{align*}
		\|\bar{X}_{j}^{(k)} - X(t_j)\|&\le \|X^{(k)}(t_j) - X(t_j)\|+ \|\bar{X}_{j}^{(k)} - X^{(k)}(t_j)\|,\\
		&\le C  \frac{L^k(\varepsilon t_j)^{k+1}}{(k+1)!}e^{(\alpha +\varepsilon L)t_j} + \frac{\varepsilon h^3 C_2}{12} e^{\alpha t_{j-1}} \frac{(1 + \varepsilon L h)^{j+k - 1} - 1}{\varepsilon L h}.\qedhere
	\end{align*}

\end{proof}

\section{Benchmark problems and experimental details}\label{app:details}

\subsection{Problems}\label{app:problems}

The five forward benchmarks are those of NeuSA~\citep{bizzi-neusa}, with the same discretisations, networks and reference solutions; we restate them here for completeness, together with the choices specific to \method{} on each. KdV and the inverse problem are defined in Sections~\ref{sec:kdv} and~\ref{sec:inverse}.

\paragraph{Wave equation: 2D Layers, 3D Layers and Marmousi.}
The acoustic wave equation in a heterogeneous medium,
\begin{equation}\label{eq:wave}
    \partial_{tt}u = c(\x)^2\,\Delta u,\qquad u(0,\x)=\exp\!\left(-\frac{|\x-\x_0|^2}{2\sigma^2}\right),\qquad \partial_t u(0,\x)=0,
\end{equation}
in three settings: two dimensions with three horizontal layers of speeds $1$, $1.25$ and $1.5$ (2D Layers, $\sigma=0.1$), two dimensions with the Marmousi velocity model~\citep{versteeg1994marmousi} ($c$ between $0.57$ and $2.18$), and three dimensions with two layers of speeds $1$ and $0.5$ (3D Layers, $\sigma=0.15$). As in NeuSA, the 2D problems are posed on $[-4,4]^2$, twice the region of interest $[-2,2]^2$, in a cosine basis, so that reflections from the artificial boundary do not reach the scored region within $T=2$; they use $201^2$ modes and 201 time steps. The 3D problem is posed on $[-1.5,1.5]^3$ with $101^3$ modes and 111 time steps up to $T=1$.\footnote{These are the discretisations of the authors' released scripts, which reproduce the numbers of their paper (Appendix~\ref{app:jax}), except for Marmousi: there the released script uses $401^2$ modes and 5{,}000 steps, whereas the paper reports $201^2$ modes and 2{,}000 steps, and every row of Table~\ref{tab:main} uses the paper's $201^2$ configuration.} The vector field is NeuSA's first-order system in $(\hat u,\hat v)$, $v=\partial_tu$: the linear part is the homogeneous wave operator with unit speed, $A=\begin{psmallmatrix}0&I\\ \Delta&0\end{psmallmatrix}$, and the network is left with the whole effect of the heterogeneity, $(c^2-1)\Delta u$, with $\epsilon=1$ (2D Layers), $2$ (Marmousi) and $10$ (3D Layers). Two settings are specific to \method{}. On 2D and 3D Layers the training rollout uses the exponential trapezoidal rule on a coarser grid than the evaluation one (101 of the 201 nodes, and 56 of the 111). On Marmousi, whose learned term is not weak, two warm-started sweeps per step do not track the trajectory; we use six, with 2{,}500 Adam steps under a warm-up--cosine learning-rate schedule (linear warm-up to $10^{-2}$ over 150 steps, cosine decay to $10^{-4}$) and the residual imposed on a window 15 grid points wider than the scored one, as in the authors' script. The references are the authors' data: eighth-order finite differences for the 2D problems and their pseudo-spectral RK4 generator for the 3D one.

\paragraph{Sine-Gordon.}
NeuSA's 1D sine-Gordon problem,
\begin{equation}\label{eq:sg}
    \partial_{tt}u = \partial_{xx}u - 10\sin u,\qquad u(0,x)=\frac{1}{\sqrt{2\pi}\sigma}\exp\!\left(-\frac{x^2}{2\sigma^2}\right),\qquad \partial_tu(0,x)=0,
\end{equation}
on $[-4,4]\times[0,3]$ with homogeneous Dirichlet conditions and $\sigma=0.1$, in a sine basis with 201 modes and 201 time steps, trained for 1{,}000 steps with $\epsilon=0.1$. NeuSA initialises around the Laplacian. For \method{} we move the mass term into the linear part, so that $A$ is the Klein--Gordon operator $\partial_{xx}-10$, whose mode-$k$ block has eigenvalues $\pm i\sqrt{k^2\pi^2/L^2+10}$, and the network is left with $10(u-\sin u)$, which is cubic at small amplitude (Section~\ref{sec:training}). The reference is the authors' pseudo-spectral RK4 solution.

\paragraph{2D Burgers.}
The coupled system
\begin{equation}\label{eq:burgers2d}
    \partial_t\mathbf{u} = -(\mathbf{u}\cdot\nabla)\mathbf{u} + \nu\Delta\mathbf{u},\qquad \mathbf{u}(0,\x)=\big(\sin(\pi x_1)\sin(\pi x_2),\,\cos(\pi x_1)\cos(\pi x_2)\big),
\end{equation}
with $\mathbf{u}=(u,v)$, $(\x,t)\in[0,4]^2\times[0,1]$, $\nu=0.01$ and periodic boundary conditions, in a Fourier basis with $201^2$ modes and 201 time steps. The linear part is the heat operator $\nu\Delta$, as in NeuSA's initialisation, and $\epsilon=0.1$. NeuSA is trained for 200 steps, as in its paper; \method{} for 1{,}000, since the loss was still decreasing at that budget and the extra steps are cheap. The reference is the authors' pseudo-spectral RK4 solution.

\paragraph{KdV and inverse Burgers.}
For KdV (Section~\ref{sec:kdv}) the linear part $-\partial_{xxx}$ couples each $(\cos_k,\sin_k)$ pair of the real Fourier basis through a $2\times2$ rotation generator of frequency $k^3$, which is exponentiated in closed form; both NeuSA and \method{} use 81 modes and 1{,}001 time steps, the finest grid on which RK4 is stable at $\Delta t=10^{-3}$ ($\Delta t\,k_{\max}^3\approx2.0$; at 101 modes it is $3.9$ and RK4 returns NaN), $\epsilon=0.1$, and \method{} is trained for 10{,}000 steps with decay $0.9995$ per step instead of $0.999$. The reference is a pseudo-spectral RK4 solution on the model's grid, which agrees with the exact soliton to $4\times10^{-4}$. For the inverse problem (Section~\ref{sec:inverse}) the linear part is $\lambda_2\Delta$ with the current value of $\lambda_2$; the reference is the clean solution with the true coefficients from NETD's pseudo-spectral solver on an 801-node grid, and the identified equation is rolled out with the same solver on $(1,2]$.

\subsection{Settings}\label{app:settings}

Table~\ref{tab:settings} lists the discretisation and training settings of every problem. Basis, modes, time grid, network, $\epsilon$, learning rate and, unless stated, number of steps are those of NeuSA's released scripts; where the paper text and the scripts differ, we used the configuration that reproduces the numbers of the paper's tables.

\begin{table}[h]
\centering
\caption{Settings per problem. $n_t$ is the number of time nodes on $[0,T]$; ``train nodes'' the coarser grid the training rollout uses with the exponential trapezoidal rule (--: full grid); $K$ the warm-started sweeps per optimiser step; ``deploy'' the number of Picard sweeps at inference. Steps are \method{}'s, with NeuSA's in parentheses when different.}
\label{tab:settings}
\scriptsize
\setlength{\tabcolsep}{3pt}
\resizebox{\textwidth}{!}{%
\begin{tabular}{l l r r r r r r l r r r}
\toprule
Problem & Basis & Modes & $n_t$ ($T$) & $\epsilon$ & Steps & lr & $K$ & Quadrature & Train nodes & Deploy & Params\\
\midrule
2D Layers & cosine & $201^2$ & 201 (2) & 1 & 2{,}000 & $10^{-2}$ & 2 & exp.\ trapezoid & 101 & 8 & 3.2e5\\
Marmousi & cosine & $201^2$ & 201 (2) & 2 & 2{,}500 (2{,}000) & cosine & 6 & trapezoid & -- & 8 & 3.2e5\\
3D Layers & cosine & $101^3$ & 111 (1) & 10 & 4{,}000 (2{,}000) & $10^{-2}$ & 2 & exp.\ trapezoid & 56 & 8 & 6.1e4\\
2D Burgers & Fourier & $201^2$ & 201 (1) & 0.1 & 1{,}000 (200) & $5\cdot10^{-3}$ & 2 & trapezoid & -- & 8 & 8.1e5\\
Sine-Gordon & sine & 201 & 201 (3) & 0.1 & 1{,}000 & $2\cdot10^{-2}$ & 2 & trapezoid & -- & 8 & 9.7e5\\
KdV & Fourier & 81 & 1{,}001 (1) & 0.1 & 10{,}000 (1{,}000) & $10^{-2}$ & 2 & trapezoid & -- & 32 & 1.6e5\\
Inverse Burgers & Fourier & 199 & 401 (1) & 0.1 & 500 / 5{,}000 & $10^{-2}$ & 8 (cold) & exp.\ trapezoid & -- & 8 & 9.5e5\\
\bottomrule
\end{tabular}}
\end{table}

\paragraph{Hardware and software.}
All runs, ours and the baselines, were executed on machines with an NVIDIA GeForce RTX~5090 (32\,GB, 575\,W power limit) and an Intel i9-13900K. \method{} is implemented in JAX~0.11 with Equinox~\citep{kidger2021equinox} and Optax~\citep{deepmind2020optax}, CUDA~13, with the default TF32 matrix multiplications; the baselines and NeuSA run the authors' PyTorch~2.9.1 code (CUDA~12.8) with torchdyn~1.0.6, as pinned by their lock file. Training time is the wall-clock time of the optimiser loop; for JAX the ahead-of-time compilation (5--20\,s per run) is excluded. Every run records the host, GPU, library versions and commit, and no partial runs are included.

\paragraph{Baselines.}
PINN, QRes and FLS use the authors' scripts and the settings of the NeuSA paper: learning rate $10^{-3}$, initial-condition weight $10^3$ (2D Burgers: $10^2$, boundary weight 1), 20{,}000 steps with 10{,}000 residual, 1{,}000 initial and 500 boundary points resampled at every step for the 2D problems (3D Layers: 70{,}000 / 10{,}000 / 2{,}000), and 10{,}000 steps on the full $201\times201$ space-time grid for the 1D problems. NeuSA runs the authors' scripts unchanged; for KdV, which the authors do not provide, we forked their 1D Burgers script, keeping model, trainer and integrator, and verified that with the network switched off the fork and our implementation agree on the initial coefficients to $10^{-15}$, on the RK4 trajectory to $10^{-6}$ and on the loss to seven digits. Marmousi NeuSA is run at the paper's configuration ($201^2$ modes, 2{,}000 steps, residual window and subsampling as in the authors' NETD script); it reaches an rMSE of $0.27\pm0.09$ against the $0.171\pm0.035$ reported in the paper, whereas on the other four benchmarks the released code reproduces the paper's tables (Table~\ref{tab:jax}). For the inverse problem the PINN, QRes and FLS baselines are those of NETD's protocol, with $(\lambda_1,\lambda_2)$ as trainable scalars, 5{,}000 steps, data weight 10 and initial-condition weight 10.

\section{Full results}\label{app:tables}

Table~\ref{tab:full} gives the mean and standard deviation over five seeds of rMSE, rMAE and training time for the six benchmarks of Table~\ref{tab:main}. Figure~\ref{fig:all} plots training time against rMSE for every method on each benchmark, and Figure~\ref{fig:errt} shows how the error accumulates over time on three of them: both neuro-spectral models start from the exact initial condition and their error grows causally, whereas the error of the best MLP-based baseline is spread over the whole interval; on Marmousi the gap between \method{} and NeuSA widens with time. 

\begin{figure}[h]
    \centering
    \includegraphics[width=\textwidth]{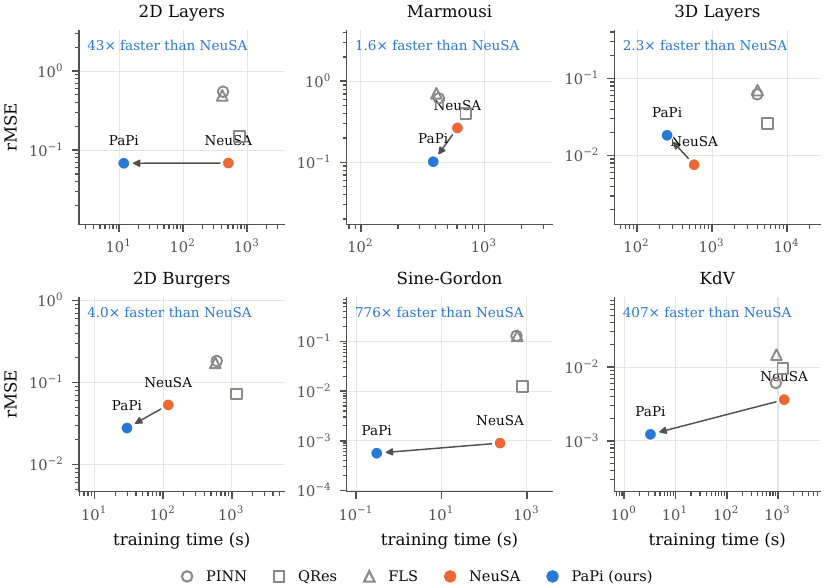}
    \caption{Training time against rMSE on every benchmark of Table~\ref{tab:main} (means over five seeds, one RTX~5090). The arrow joins NeuSA to \method{}: the same model, trained through the parallel Picard--Duhamel rollout instead of the sequential RK4 one.}
    \label{fig:all}
\end{figure}

\begin{figure}[h]
    \centering
    \includegraphics[width=\textwidth]{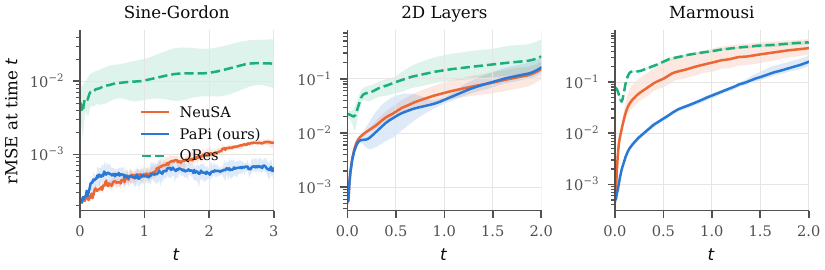}
    \caption{Relative $L^2$ error at each time $t$ (mean over five seeds, band from best to worst seed) for \method{}, NeuSA and the best MLP-based baseline, QRes.}
    \label{fig:errt}
\end{figure}

\begin{table}[h]
\centering
\caption{Main benchmarks: mean $\pm$ std over five seeds, all on the same RTX~5090.}
\label{tab:full}
\small
\begin{tabular}{l l r r r}
\toprule
Problem & Model & rMSE & rMAE & TT (s)\\
\midrule
2D Layers & PINN & 5.5e-1 $\pm$ 8.9e-2 & 7.8e-1 $\pm$ 1.3e-1 & 420 $\pm$ 56\\
 & QRes & 1.5e-1 $\pm$ 7.2e-2 & 1.9e-1 $\pm$ 7.5e-2 & 765 $\pm$ 212\\
 & FLS & 4.9e-1 $\pm$ 5.2e-2 & 7.2e-1 $\pm$ 8.3e-2 & 408 $\pm$ 60\\
 & NeuSA & 6.9e-2 $\pm$ 2.0e-2 & 9.4e-2 $\pm$ 3.1e-2 & 508 $\pm$ 19\\
 & \method{} & 6.8e-2 $\pm$ 1.8e-2 & 9.1e-2 $\pm$ 2.7e-2 & 11.8 $\pm$ 0.0\\
\midrule
Marmousi & PINN & 6.2e-1 $\pm$ 3.5e-2 & 9.2e-1 $\pm$ 8.1e-2 & 427 $\pm$ 68\\
 & QRes & 4.0e-1 $\pm$ 1.5e-2 & 5.4e-1 $\pm$ 2.0e-2 & 707 $\pm$ 181\\
 & FLS & 7.1e-1 $\pm$ 6.3e-2 & 1.06 $\pm$ 1.2e-1 & 408 $\pm$ 54\\
 & NeuSA & 2.7e-1 $\pm$ 9.4e-2 & 3.4e-1 $\pm$ 1.0e-1 & 604 $\pm$ 191\\
 & \method{} & 1.5e-1 $\pm$ 1.5e-2 & 2.0e-1 $\pm$ 1.5e-2 & 65.7 $\pm$ 0.3\\
\midrule
3D Layers & PINN & 6.2e-2 $\pm$ 6.4e-3 & 1.4e-1 $\pm$ 2.1e-2 & 4000 $\pm$ 172\\
 & QRes & 2.6e-2 $\pm$ 1.6e-2 & 4.4e-2 $\pm$ 2.1e-2 & 5496 $\pm$ 43\\
 & FLS & 7.1e-2 $\pm$ 1.5e-2 & 1.8e-1 $\pm$ 4.9e-2 & 4050 $\pm$ 23\\
 & NeuSA & 7.6e-3 $\pm$ 7.7e-4 & 1.1e-2 $\pm$ 2.7e-3 & 577 $\pm$ 13\\
 & \method{} & 1.8e-2 $\pm$ 1.1e-2 & 2.9e-2 $\pm$ 1.7e-2 & 251 $\pm$ 1\\
\midrule
2D Burgers & PINN & 1.8e-1 $\pm$ 7.2e-2 & 1.3e-1 $\pm$ 6.8e-2 & 605 $\pm$ 87\\
 & QRes & 7.3e-2 $\pm$ 1.3e-3 & 2.3e-2 $\pm$ 1.6e-3 & 1178 $\pm$ 348\\
 & FLS & 1.7e-1 $\pm$ 7.0e-2 & 1.2e-1 $\pm$ 6.8e-2 & 578 $\pm$ 68\\
 & NeuSA & 5.3e-2 $\pm$ 5.4e-3 & 4.3e-2 $\pm$ 4.3e-3 & 119 $\pm$ 4\\
 & \method{} & 2.8e-2 $\pm$ 1.7e-3 & 2.1e-2 $\pm$ 2.4e-3 & 29.6 $\pm$ 2.1\\
\midrule
Sine-Gordon & PINN & 1.3e-1 $\pm$ 2.3e-2 & 1.4e-1 $\pm$ 2.4e-2 & 569 $\pm$ 16\\
 & QRes & 1.3e-2 $\pm$ 7.1e-3 & 1.8e-2 $\pm$ 1.1e-2 & 781 $\pm$ 28\\
 & FLS & 1.3e-1 $\pm$ 3.0e-2 & 1.4e-1 $\pm$ 2.9e-2 & 594 $\pm$ 17\\
 & NeuSA & 9.0e-4 $\pm$ 1.7e-5 & 1.2e-3 $\pm$ 3.2e-5 & 234 $\pm$ 9\\
 & \method{} & 5.6e-4 $\pm$ 7.0e-5 & 8.1e-4 $\pm$ 8.1e-5 & 0.3 $\pm$ 0.0\\
\midrule
KdV & PINN & 6.1e-3 $\pm$ 7.6e-4 & 9.0e-3 $\pm$ 1.1e-3 & 908 $\pm$ 17\\
 & QRes & 9.5e-3 $\pm$ 1.0e-2 & 1.7e-2 $\pm$ 1.9e-2 & 1240 $\pm$ 29\\
 & FLS & 1.5e-2 $\pm$ 1.3e-2 & 2.9e-2 $\pm$ 3.2e-2 & 935 $\pm$ 16\\
 & NeuSA & 3.6e-3 $\pm$ 1.2e-3 & 6.5e-3 $\pm$ 1.9e-3 & 1322 $\pm$ 109\\
 & \method{} & 1.2e-3 $\pm$ 4.7e-4 & 2.2e-3 $\pm$ 7.5e-4 & 3.2 $\pm$ 0.0\\
 & \method{}, 101 modes & 1.4e-3 $\pm$ 7.7e-4 & 2.5e-3 $\pm$ 1.3e-3 & 3.4 $\pm$ 0.0\\
\bottomrule
\end{tabular}
\end{table}

\section{Number of Picard sweeps at inference}\label{app:sweeps}

Training uses two warm-started sweeps per optimiser step (six on Marmousi), but the trained model is deployed with a cold Picard--Duhamel solve starting from the linear flow $\Phi_0^tX_0$. Figure~\ref{fig:sweeps} scores every trained model after $k=1,2,\ldots$ sweeps. On all problems but KdV the error is stationary from eight sweeps on, so eight were used throughout; at four sweeps the iteration has not converged on the problems with the largest learned term (2D Layers, Marmousi, 2D Burgers). KdV, whose nonlinear term $6u\partial_xu$ is not small relative to the horizon, needs 32 sweeps; the cold iteration converges between 16 and 32 sweeps and does not change afterwards. This is consistent with Theorem~\ref{thm:convergence}: the number of sweeps required grows with $\varepsilon kT$ and is independent of the stiffness of $A$. A sweep costs a fraction of a millisecond on the 1D problems and up to a few hundred milliseconds on Marmousi, so inference is dominated by the 8 to 32 parallel sweeps rather than by 201 sequential RK4 steps.

\begin{figure}[h]
    \centering
    \includegraphics[width=0.62\textwidth]{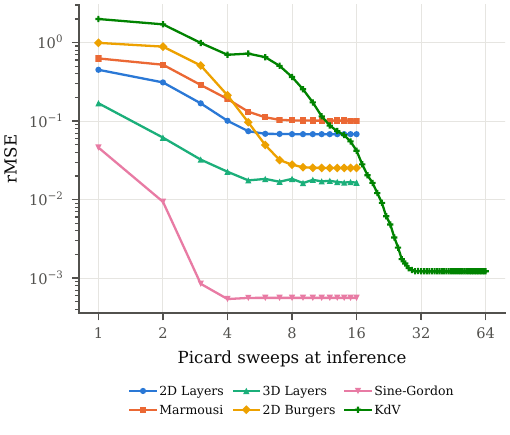}
    \caption{rMSE of the trained \method{} models after $k$ cold Picard--Duhamel sweeps from the linear flow, mean over five seeds. Every problem but KdV is converged by eight sweeps.}
    \label{fig:sweeps}
\end{figure}

\section{JAX implementation and comparison with the PyTorch RK4 rollout}\label{app:jax}

\paragraph{The port.}
\method{} is implemented in a JAX library that also contains a port of NeuSA: the same spectral bases (analysis and synthesis matrices, derivative operators), the same dimension-wise MLP, the same first-order $(\hat u,\hat v)$ formulation of the second-order problems, the same residual losses and metrics, and an explicit RK4 rollout. Two things differ from the PyTorch original. First, the basis matrices are assembled in double precision and then cast, which reduces the round-trip error of the transforms at 201 modes by a factor of six to ten; second, the RK4 rollout is a single \texttt{lax.scan} over the time grid and the optimiser steps are fused into scanned chunks, instead of torchdyn's Python loop over stages and steps. The arithmetic is otherwise identical. Equivalence with the original is checked by a parity test suite: with identical weights copied across, the two implementations agree on the basis matrices, the projected initial condition, the RK4 trajectory, the residual and its gradient, and the 2D pipeline is pinned by scoring the NeuSA authors' released Marmousi checkpoint at its reported error. The parallel scans of \method{} are the \texttt{associative\_scan} primitive of JAX,
with the propagator applied mode by mode.

\paragraph{How much of the speedup is the framework?}
Table~\ref{tab:jax} compares, on every problem and the same hardware, the PyTorch NeuSA (torchdyn RK4), the JAX RK4 rollout and \method{}, with the time per optimiser step, since the step counts differ on some problems. The JAX RK4 rollout is $14$--$20$ times faster per step than torchdyn on the 1D problems, whose per-step cost in PyTorch is dominated by the roughly $4n_t$ dependent kernel launches of the rollout ($800$ for sine-Gordon, $4{,}000$ for KdV), and $1$--$5$ times faster on the 2D and 3D problems, where each RK4 stage is a large tensor contraction and the framework overhead is negligible ($1.7\times$ on 3D Layers, $4$--$5\times$ on the 2D problems). The remaining factor, RK4 in JAX against \method{} in JAX, isolates the integrator: $56\times$ on sine-Gordon, $10\times$ on 2D Layers, $4\times$ on 2D Burgers, and $2$--$3\times$ on Marmousi and 3D Layers, where a Picard sweep costs a sizeable fraction of an RK4 rollout because the state is large enough for the GPU to be saturated either way. On KdV the per-step factor is $210\times$, the largest of all, because the stiff dispersion forces the RK4 rollout through 1{,}001 steps while the exact propagator makes the cost of a sweep independent of the stiffness.

The JAX RK4 rollout also reproduces the accuracy of the PyTorch NeuSA within seed variability (Table~\ref{tab:jax}), and on the four benchmarks where the authors' code reproduces their paper, both do; the JAX rollout is slightly more accurate on sine-Gordon, which we attribute to the double-precision basis assembly.

\begin{table}[h]
\centering
\caption{Same model, three integrators, same RTX~5090: mean over five seeds of rMSE, optimiser steps, training time and time per step. ``NeuSA'' is the authors' PyTorch code (torchdyn RK4); ``RK4 (JAX)'' our port of the same rollout; both ratios are of times per step.}
\label{tab:jax}
\scriptsize
\setlength{\tabcolsep}{3pt}
\begin{tabular}{l rrr rrr rrr rr}
\toprule
& \multicolumn{3}{c}{NeuSA (PyTorch)} & \multicolumn{3}{c}{RK4 (JAX)} & \multicolumn{3}{c}{\method{} (JAX)} & \multicolumn{2}{c}{speedup / step}\\
\cmidrule(lr){2-4}\cmidrule(lr){5-7}\cmidrule(lr){8-10}\cmidrule(lr){11-12}
Problem & rMSE & steps & TT (s) & rMSE & steps & TT (s) & rMSE & steps & TT (s) & torch/JAX & RK4/\method{}\\
\midrule
2D Layers & 6.9e-2 & 2{,}000 & 508 & 9.1e-2 & 2{,}000 & 117 & 6.8e-2 & 2{,}000 & 11.8 & 4.4 & 9.9\\
Marmousi & 2.7e-1 & 2{,}000 & 604 & 1.9e-1 & 5{,}000 & 290 & 1.5e-1 & 2{,}500 & 66 & 5.2 & 2.2\\
3D Layers & 7.6e-3 & 2{,}000 & 577 & 5.3e-3 & 4{,}000 & 676 & 1.8e-2 & 4{,}000 & 251 & 1.7 & 2.7\\
2D Burgers & 5.3e-2 & 200 & 119 & 2.5e-2 & 1{,}000 & 123 & 2.8e-2 & 1{,}000 & 29.6 & 4.8 & 4.2\\
Sine-Gordon & 9.0e-4 & 1{,}000 & 234 & 7.3e-4 & 1{,}000 & 16.7 & 5.6e-4 & 1{,}000 & 0.3 & 14 & 56\\
KdV & 3.6e-3 & 1{,}000 & 1322 & 3.3e-3 & 1{,}000 & 67.0 & 1.2e-3 & 10{,}000 & 3.2 & 20 & 210\\
\bottomrule
\end{tabular}
\end{table}

\paragraph{Stability.}
The explicit rollout constrains the discretisation in a way \method{} does not. On the padded Marmousi grid the largest eigenvalue of the wave operator gives $\Delta t\,\omega_{\max}\approx2.2$ at NeuSA's $\Delta t=0.01$, just inside RK4's stability limit of about $2.8$ on the imaginary axis; a coarser time grid produces NaN rather than a worse answer. On KdV the same limit caps the number of modes at 81 for $\Delta t=10^{-3}$, and doubling the resolution requires eight times more steps. For the inverse Burgers problem the viscous term gives $\Delta t\,\lambda_{\max}\approx24$ on the data grid, so the RK4 baseline of our port has to sub-step the rollout by a factor of ten to twenty. In all three cases the exact propagator of \method{} is unconditionally stable, and its time grid is chosen to resolve the nonlinear correction alone; with the exponential trapezoidal rule the training grid can even be coarser than the evaluation grid, as on 2D and 3D Layers.
% --- End Inserted File ---

% ---- Inserted File ----
\section{Sine-Gordon ablation study}
\label{app:pd-ablation}
Two parts of our method (the solver and training) are explored in this section considering a single controlled problem (the Sine-Gordon equation). We observe that exact linear propagation is what makes the fixed-point iteration converge at all, and training accuracy is set mainly by the step size, with Adam-type methods at $\eta_0\approx0.01$ reaching errors below $10^{-3}$.

\subsection{Setup}
\label{app:pd-setup}
We solve $u_{tt}=u_{xx}-10\sin u$ on $[-4,4]\times[0,3]$ with $u(\pm4,t)=0$, $u_t(x,0)=0$, and a Gaussian pulse $u(x,0)=e^{-x^2/(2\sigma^2)}/(\sqrt{2\pi}\sigma)$, $\sigma=0.1$, which splits into two counter-propagating waves that the nonlinearity keeps reshaping. A sine expansion in $M=201$ modes enforces the boundary conditions; the nonlinearity is evaluated at the nodes without dealiasing. Writing $y=(a,\dot a)$ for the coefficients and velocities gives $\dot y=Ly+N(y)$, with wave operator $L$ and nonlinear force $N$. We report the relative Frobenius error $E=\|U-U_\star\|_{\mathrm F}/\|U_\star\|_{\mathrm F}$ over the time--space grid against a reference $U_\star$ from adaptive eighth-order Dormand-Prince integration of the same system (accurate to $6\times10^{-12}$). Since the spatial discretization is shared, $E$ measures only time-integration and, later, learning error.

\subsection{Solver: exact linear propagation drives convergence}
\label{app:pd-solver}
The two fixed-point iterations we compare differ in what they leave to quadrature. Plain Picard iteration integrates the whole right-hand side, whereas Picard--Duhamel (PD) propagates the linear part exactly through $\Phi(t)=e^{tL}$ and iterates only on the nonlinear force:
\begin{equation}
 y_{\rm P}^{k+1}(t)=y_0+\!\int_0^t\![Ly_{\rm P}^k+N(y_{\rm P}^k)]\,ds,
 \qquad
 y_{\rm PD}^{k+1}(t)=\Phi(t)y_0+\!\int_0^t\!\Phi(t-s)N(y_{\rm PD}^k(s))\,ds,
 \label{eq:pd-updates}
\end{equation}
starting from $y_0$ and $\Phi(t)y_0$, respectively. A single PD sweep is the classical Duhamel approximation. In both cases the integrals are evaluated with the trapezoidal rule and parallel prefix sums, without damping or time windowing. We run both iterations on 101, 201, and 401 time nodes for $K\in\{1,\dots,32\}$ sweeps.

\begin{table}[h]
\centering\small
\caption{Only PD converges. Relative error $E$ at fixed sweep budgets; times
are medians of 20 warmed float64 GPU solves, excluding compilation.}
\label{tab:pd-results}
\begin{tabular}{rrrrr}
\toprule
Time nodes & Picard, $K=16$ & Duhamel, $K=1$ & PD, $K=16$ & PD time (ms)\\
\midrule
101 & $7.55\times10^{17}$ & 5.025 & $8.63\times10^{-4}$ & 3.50\\
201 & $7.04\times10^{17}$ & 5.021 & $2.16\times10^{-4}$ & 5.10\\
401 & $6.88\times10^{17}$ & 5.018 & $5.39\times10^{-5}$ & 7.76\\
\bottomrule
\end{tabular}
\end{table}

\begin{figure}[t]
\centering
\includegraphics[width=\linewidth]{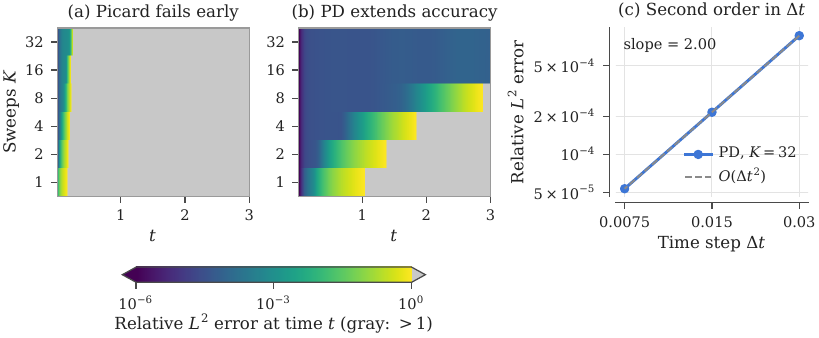}
\caption{Each PD sweep extends the accurate time window, while Picard fails
early at every budget. (a,b) Error at each time $t$ on the 401-node grid
(gray: above one); the bottom PD row is the Duhamel solution. (c) Once
converged ($K=32$), PD error falls at exactly second order in $\Delta t$.}
\label{fig:pd-convergence}
\end{figure}

The outcome is not close (Table~\ref{tab:pd-results}). Plain Picard never produces a usable solution at any of these budgets; by $K=16$ its error has grown past $10^{17}$. A single Duhamel sweep at least stays bounded, but it is still off by a factor of five. Repeating the Duhamel step is what closes the gap: on the 401-node grid, PD reaches $E=5.39\times10^{-5}$ in 7.8\,ms. Figure~\ref{fig:pd-convergence} shows why. Picard loses accuracy almost as soon as it starts, whereas each PD sweep pushes the accurate region further along the time axis, until after 16 sweeps it covers the whole horizon. Beyond that point additional sweeps only move the fifth significant digit, and what error remains is due to the quadrature: halving $\Delta t$ cuts it by a factor of four, with observed orders of 2.002 and 2.000. Nor is this a matter of PD starting from a better guess: crossing both updates with both starting paths, PD from the constant path reaches the same $5.39\times10^{-5}$, and Picard from the free wave still fails (Figure~\ref{fig:pd-validation-numerics}a).

\subsection{Training: step size dominates the choice of optimizer}
\label{app:pd-training}
For the training study we replace the exact force by a learned one, $0.1N_\theta(a)$, where $N_\theta$ is a two-layer LeakyReLU network of width 804 with 971{,}433 parameters. Training rolls the model out with RK4 on 201 time nodes and minimizes the mean squared residual of $u_{tt}-u_{xx}+10\sin u$ along that rollout, optionally with a penalty $\tfrac{\lambda}{2}\|\theta\|_2^2$ added; gradients flow through the entire rollout. Every run consists of 1,000 full-batch float32 updates and is repeated with three initialization seeds. We report the mean and sample standard deviation over seeds, which should be read as the spread due to initialization rather than as a confidence interval. The comparison covers Adam, AdamW (weight decay $10^{-4}$), RMSProp, momentum SGD, and Muon at initial learning rates $\eta_0\in\{0.001,0.01,0.02\}$ with decay $0.999^k$; Polyak SGD with step caps $\{0.01,0.1,1\}$; and L-BFGS with a zoom line search. A separate sweep adds $\lambda\in\{10^{-6},10^{-4},10^{-2}\}$ to Adam at $\eta_0=0.01$. Altogether this gives 22 configurations and 66 runs.

\begin{figure}[t]
\centering
\includegraphics[width=\linewidth]{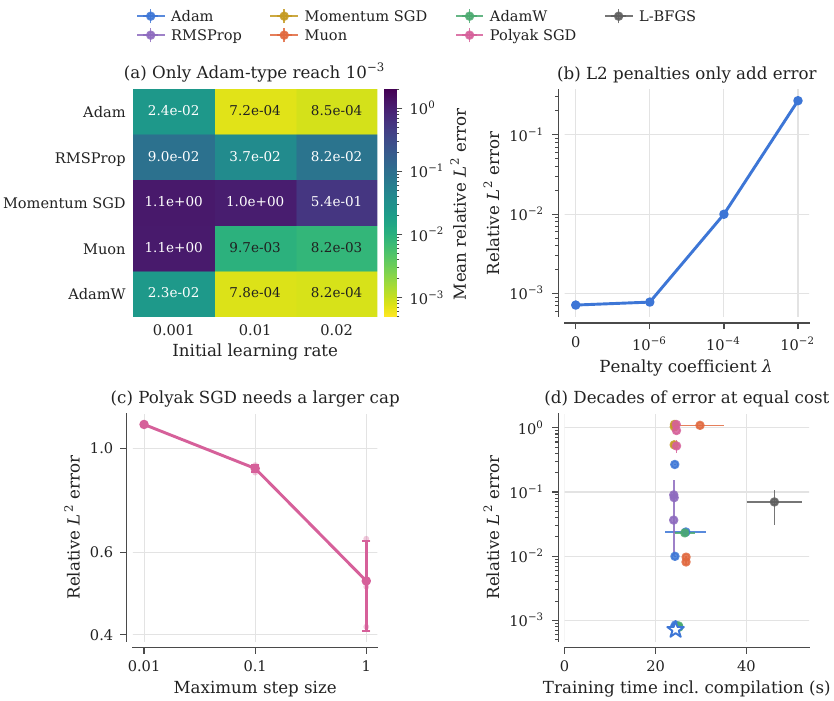}
\caption{Step size and optimizer family determine accuracy; regularization only
adds error. (a) Mean error per optimizer and initial learning rate. (b) L2
penalty for Adam at $\eta_0=0.01$. (c) Polyak SGD step-size cap. (d) All
configurations against training time; the star marks Adam at $\eta_0=0.01$.
Error bars are standard deviations over three seeds.}
\label{fig:pd-training-sensitivity}
\end{figure}

\paragraph{Step size matters more than the optimizer family.} Only the Adam-type methods get below $10^{-3}$, and they do so only once the learning rate is large enough (Figure~\ref{fig:pd-training-sensitivity}a). The most accurate configuration is Adam at $\eta_0=0.01$, with $E=(7.17\pm0.27)\times10^{-4}$; AdamW lands within its seed spread. Drop the rate to $\eta_0=0.001$ and both become thirty times worse. Muon reaches $8.2\times10^{-3}$ and RMSProp $3.7\times10^{-2}$, while momentum SGD and Polyak SGD never get below $0.5$ and improve only as their step is allowed to grow (panel c). L-BFGS stalls around $7\times10^{-2}$, and in float32 22\% of its line searches terminate without satisfying both acceptance criteria.

\paragraph{Explicit regularization only hurts.} Since the objective is a physics residual on a single trajectory, there is nothing to overfit, and an L2 penalty can only pull the solution away from the one that satisfies the equation. That is what happens: $\lambda=10^{-6}$, $10^{-4}$, and $10^{-2}$ raise the error by factors of 1.09, 13.9, and 374 (Figure~\ref{fig:pd-training-sensitivity}b). The decoupled weight decay of $10^{-4}$ in AdamW is mild enough to do no harm.

\paragraph{Cost is nearly uniform.} All first-order methods take about 24\,s per run (23--26\,ms per update), so the three orders of magnitude separating the configurations in Figure~\ref{fig:pd-training-sensitivity}(d) are not bought with compute. The exception is L-BFGS, whose line search costs 1,662 gradient evaluations and 46\,s.

\subsection{Robustness checks}
\label{app:pd-validation}
Three questions remain: how much of the training error is really integration error, whether the weaker optimizers were merely under-tuned or under-trained, and whether the L-BFGS failures are a float32 artifact. Each gets a targeted control on the same problem and architecture.

\begin{figure}[t]
\centering
\includegraphics[width=0.9\linewidth]{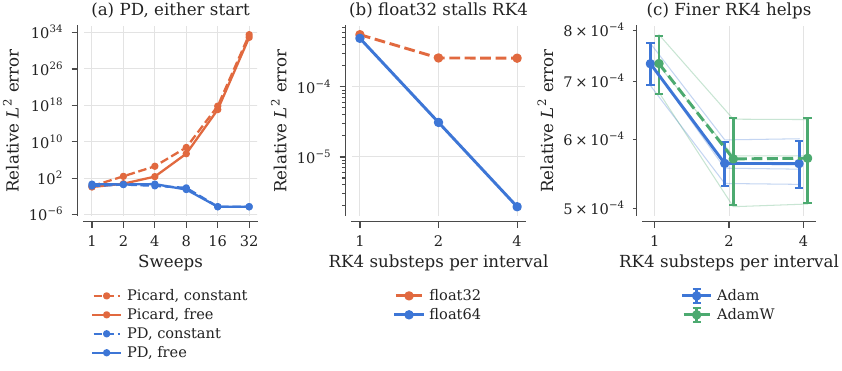}
\caption{The PD gain comes from the update, and part of the training error comes from time integration. (a) Both updates from both initial paths (401 nodes). (b) Exact-force RK4 in float32 and float64. (c) Frozen Adam and AdamW models evaluated with finer RK4. Panels (b,c) score the original 201 nodes.}
\label{fig:pd-validation-numerics}
\end{figure}

\begin{figure}[t]
\centering
\includegraphics[width=0.9\linewidth]{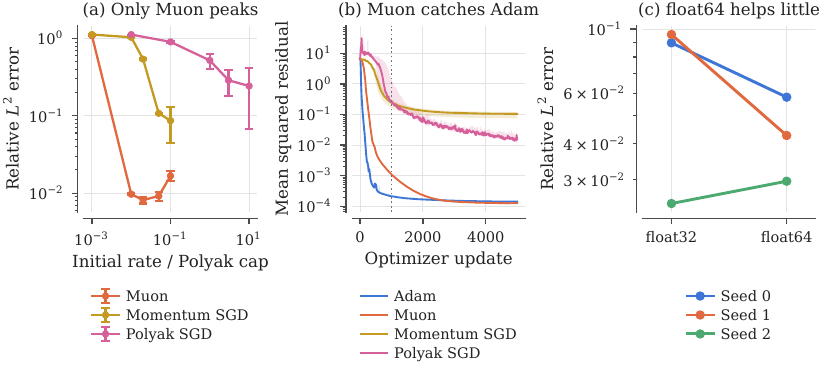}
\caption{Given enough updates Muon matches Adam, while float64 does not rescue L-BFGS. (a) Extended step-size ranges at 1,000 updates. (b) Residuals continued to 5,000 updates: seed median with min--max band, smoothed by a 51-update running median (dotted line: original budget). (c) L-BFGS in float32 and float64 from identical initial parameters.}
\label{fig:pd-validation-optimization}
\end{figure}

\paragraph{Integration error.} Even with the exact force, RK4 on the training grid already incurs $E=4.93\times10^{-4}$ in float64. Refining to four substeps per interval brings this down to $1.94\times10^{-6}$, but in float32 the same refinement stalls at $2.6\times10^{-4}$ (Figure~\ref{fig:pd-validation-numerics}b). The learned models are affected the same way: freezing the trained Adam and AdamW models and evaluating them with finer RK4 lowers their error from $7.33\times10^{-4}$ to $5.6$--$5.7\times10^{-4}$ (panel c). A good part of the reported training error is therefore integration error, not a deficiency of the learned force.

\paragraph{Tuning and budget.} To check whether the weaker methods were held back by their step-size grid or by the training budget, we extended the grids, picked the best setting for each method by its final training residual (never by the reference error), and continued training to 5,000 updates (Figure~\ref{fig:pd-validation-optimization}a,b). For Muon the optimum lies inside the grid at $\eta_0=0.02$, and with the longer budget it catches up with Adam, at $5.96\times10^{-4}$ against $6.27\times10^{-4}$. Momentum SGD and Polyak SGD, by contrast, keep improving right up to the largest steps we tested and still only reach $2$--$4\times10^{-2}$. The ranking at 1,000 updates therefore says as much about how fast each method converges as about where it ends up.

\paragraph{Precision.} Rerunning L-BFGS in float64 from the same initial parameters eliminates the line-search failures entirely, yet the error only improves to $(4.35\pm1.42)\times10^{-2}$, and for one seed it gets worse (Figure~\ref{fig:pd-validation-optimization}c). Precision explains the failed line searches, but not why L-BFGS does poorly here.

\subsection{Conclusion}
\label{app:pd-conclusion}
Exact linear propagation is the decisive design choice. It costs a few milliseconds, and it turns an iteration that fails at every budget we tried into one that converges in 16 sweeps to the quadrature floor, leaving only the second-order error set by the time step. On the training side, Adam or AdamW at $\eta_0\approx0.01$ with no explicit regularization is the most accurate choice at the standard budget, and the gaps between optimizers say more about how quickly they converge than how far they can go, as Muon's catching up at 5,000 updates shows. The best learned models settle around $6\times10^{-4}$, the same order as the exact-force RK4 error on the training grid ($4.9\times10^{-4}$) and ten times the converged PD solver. Further gains are therefore more likely to come from the time integrator used during training than from more optimizer tuning, which argues for pairing the learned force with the Picard--Duhamel solver itself.
% --- End Inserted File ---

\section{Other implementation details}

\subsection{Parallel Scans}\label{app:parallel_scans}

A parallel scan is an efficient way of evaluating a recurrence of the form
\begin{equation}
	h_t = A_t h_{t-1} + b_t,
	\qquad t=1,\ldots,T,
\end{equation}
without explicitly stepping through the sequence one element at a time. The key observation is that each step can be represented as an affine transformation
\[
h \mapsto A_t h+b_t,
\]
and that affine transformations can be composed associatively. Writing
\[
(A_2,b_2)\circ(A_1,b_1)
=
(A_2A_1,;A_2b_1+b_2),
\]
the entire sequence can therefore be evaluated using a standard parallel prefix-scan algorithm.

Conceptually, the sequential computation
\[
h_0 \longrightarrow h_1 \longrightarrow h_2
\longrightarrow h_3 \longrightarrow \cdots \longrightarrow h_T
\]
is replaced by a tree of compositions:
\begin{equation}
	\begin{array}{cccccccc}
		x_1 & x_2 & x_3 & x_4 & x_5 & x_6 & x_7 & x_8 \\[2mm]
		\downarrow & \downarrow & \downarrow & \downarrow
		& \downarrow & \downarrow & \downarrow & \downarrow \\[2mm]
		\multicolumn{2}{c}{x_1\circ x_2}
		&
		\multicolumn{2}{c}{x_3\circ x_4}
		&
		\multicolumn{2}{c}{x_5\circ x_6}
		&
		\multicolumn{2}{c}{x_7\circ x_8}
		\\[4mm]
		\multicolumn{4}{c}{\searrow\hspace{1.5cm}\searrow}
		&
		\multicolumn{4}{c}{\searrow\hspace{1.5cm}\searrow}
		\\[-1mm]
		\multicolumn{4}{c}{(x_1\circ x_2)\circ(x_3\circ x_4)}
		&
		\multicolumn{4}{c}{(x_5\circ x_6)\circ(x_7\circ x_8)}
		\\[4mm]
		\multicolumn{8}{c}{\searrow\hspace{5cm}\searrow}
		\\[2mm]
		\multicolumn{8}{c}{x_1\circ x_2\circ x_3\circ x_4
			\circ x_5\circ x_6\circ x_7\circ x_8}
	\end{array}
\end{equation}so that the computational depth is reduced from $\mathcal{O}(T)$ to $\mathcal{O}(\log T)$, at the cost of $\mathcal{O}(T)$ total work.

This structure is particularly well suited to GPUs. Although the scan performs more operations than a strictly sequential recurrence, many of these operations are independent and can be executed simultaneously across thousands of GPU threads. Moreover, the computation consists primarily of regular tensor operations with predictable memory access patterns, allowing modern GPU kernels to exploit massive parallelism and high memory bandwidth. In practice, this can make a parallel scan substantially faster than a sequential recurrence for long sequences, even when the underlying recurrence itself is unchanged.

The same idea extends beyond scalar or diagonal recurrences to more general associative state-update operators. This is one reason parallel scans are useful for state-space models and other sequence architectures: the temporal structure is retained, while the sequential dependency that normally prevents efficient GPU utilization is replaced by an associative reduction. 

\subsection{Exponential trapezoidal quadrature}
\label{app:exponential_trapezoidal}

The trapezoidal rule introduced in \eqref{eq:trapezoidal_method_def} approximates the integrand in \eqref{eq:duhamel-scheme}, $\Phi_s^t F(X^{(k)}(s))$, by a piecewise-linear function. However in the setting when $\Phi_s^t$ has a closed formula, for instance when $A(t)=A$ is constant and $\Phi_s^t  = \exp(A (t-s))$, it is often preferable to instead approximate $F(X^{(k)}(s)$ by a piecewise-linear function $\iota(s)$, and evaluate $\int_0^t
\Phi_s^t \iota(s)\,\mathrm{d}s$ exactly. This leads to the exponential trapezoidal rule. The construction below follows the standard framework of
exponential quadrature; see, for example,
\cite{hochbruck2010exponential}.

We first recall the $\varphi$-functions used in exponential integration. We may define them by their Taylor expansion
$$
\varphiFun_{\ell}(z):=\sum_{n=0}^{\infty} \frac{z^{n}}{(n+\ell)!}
$$
With most relevance to our setting are the first three $\varphi$-functions,
\[
\varphiFun_0(z)=e^z,
\qquad
\varphiFun_1(z)=\frac{e^z-1}{z},
\qquad
\varphiFun_2(z)=\frac{e^z-1-z}{z^2}.
\]
The Taylor expansion allows to define the $\varphi$ functions also for matrices, in a similar way to the definition of matrix exponential.

Now let $\Phi_s^t  = \exp(A (t-s))$ then it the definition of $\varphi$ implies that
\begin{equation}
\int_{0}^{t}\Phi_s^t (a (t-s) + b s)\,\mathrm{d}s = \left(\varphiFun_1(tA)-\varphiFun_2(tA)\right) a
+
\varphiFun_2(A)b.
\label{eq:exponential_integral}
\end{equation}
Letting $P=P_j = e^{h A}$, $B_0= \varphiFun_1(hA)-\varphiFun_2(hA)$ and $B_1= \varphiFun_2(hA)$,  \eqref{eq:exponential_integral} leads to the exponential trapezoidal scheme, analog to \eqref{eq:trapezoidal_method_def}:
\begin{align}
	\label{eq:exponential_trapezoidal_global}
	\bar X_{\ET,j}^{(k+1)}
	:=&
	\Phi_{0}^{t_j}X_0
	+\varepsilon h
	\sum_{i=0}^{j-1}
	\Phi_{t_{i+1}}^{t_j}
	\left(B_0g_{\ET,i}^{(k)}+B_1g_{\ET,i+1}^{(k)}\right),
\end{align}
where $g_{\ET,i}^{(k)} = F(\bar{X}_{\ET,i}^{(k)})$. 
Finally, and most crucially, this integration scheme also enjoys a recursion similar to that of the trapezoidal method, that allows it to be calculated with a parallel scan:
\begin{equation}
	\label{eq:exponential_trapezoidal_recurrence}
	\overline X_{\ET,j+1}^{(k+1)}
	=
	P\overline X_{\ET,j}^{(k+1)}
	+b_{\ET,j}^{(k)},\quad\text{where}\qquad 
	b_{\ET,j}^{(k)}
	:=
	\varepsilon
	\left(B_0g_{ET,j}^{(k)}+B_1g_{ET,j+1}^{(k)}\right).
\end{equation}


\begin{thebibliography}{34}
\providecommand{\natexlab}[1]{#1}
\providecommand{\url}[1]{\texttt{#1}}
\expandafter\ifx\csname urlstyle\endcsname\relax
  \providecommand{\doi}[1]{doi: #1}\else
  \providecommand{\doi}{doi: \begingroup \urlstyle{rm}\Url}\fi

\bibitem[Balestro et~al.(2026)Balestro, Marques, Mendon{\c{c}}a, Moreira,
  de~Oliveira, Ganacim, Novello, Petrov, Yukimura, and
  Nissenbaum]{balestro2026neurospectral}
Vitor Balestro, M{\'a}rcio Marques, Leonardo Mendon{\c{c}}a, Leonardo~M.
  Moreira, Christian~J{\'u}nior de~Oliveira, Francisco Ganacim, Tiago Novello,
  Pavel Petrov, Daniel Yukimura, and Lucas Nissenbaum.
\newblock Neuro-spectral architectures with time-domain decomposition.
\newblock In \emph{ICLR 2026 Workshop on AI and Partial Differential
  Equations}, 2026.

\bibitem[Bizzi et~al.(2025)Bizzi, Moreira, Marques, Mendon\c{c}a, de~Oliveira,
  Balestro, dos Santos~Fernandez, Yukimura, Petrov, Pereira, Novello, and
  Nissenbaum]{bizzi-neusa}
Arthur Bizzi, Leonardo Moreira, M\'{a}rcio Marques, Leonardo Mendon\c{c}a,
  Christian de~Oliveira, Vitor Balestro, Lucas dos Santos~Fernandez, Daniel
  Yukimura, Pavel Petrov, Jo\~{a}o Pereira, Tiago Novello, and Lucas
  Nissenbaum.
\newblock Neuro-spectral architectures for causal physics-informed networks.
\newblock In D.~Belgrave, C.~Zhang, H.~Lin, R.~Pascanu, P.~Koniusz,
  M.~Ghassemi, and N.~Chen (eds.), \emph{Advances in Neural Information
  Processing Systems}, volume 38, Main Conference, pp.\  54419--54449. Curran
  Associates, Inc., 2025.
\newblock \doi{10.52202/085713-1818}.
\newblock URL
  \url{https://proceedings.neurips.cc/paper_files/paper/2025/file/4e52bbb99690d1e05c7ef7b4c8b3569a-Paper-Conference.pdf}.

\bibitem[Boyd(2013)]{boyd2001spectralmethods}
J.P. Boyd.
\newblock \emph{Chebyshev and Fourier Spectral Methods: Second Revised
  Edition}.
\newblock Dover Books on Mathematics. Dover Publications, 2013.
\newblock ISBN 9780486141923.
\newblock URL \url{https://books.google.com.br/books?id=b4TCAgAAQBAJ}.

\bibitem[Bradbury et~al.(2018)Bradbury, Frostig, Hawkins, Johnson, Katariya,
  Leary, Maclaurin, Necula, Paszke, Vander{P}las, Wanderman-{M}ilne, and
  Zhang]{jax2018github}
James Bradbury, Roy Frostig, Peter Hawkins, Matthew~James Johnson, Yash
  Katariya, Chris Leary, Dougal Maclaurin, George Necula, Adam Paszke, Jake
  Vander{P}las, Skye Wanderman-{M}ilne, and Qiao Zhang.
\newblock {JAX}: composable transformations of {P}ython+{N}um{P}y programs,
  2018.
\newblock URL \url{http://github.com/jax-ml/jax}.

\bibitem[Bu \& Karpatne(2021)Bu and Karpatne]{bu2021qres}
Jie Bu and Anuj Karpatne.
\newblock Quadratic residual networks: A new class of neural networks for
  solving forward and inverse problems in physics involving pdes.
\newblock In \emph{Proceedings of the 2021 SIAM International Conference on
  Data Mining (SDM)}, pp.\  675--683. SIAM, 2021.

\bibitem[Canuto et~al.(2006)Canuto, Hussaini, Quarteroni, and
  Zang]{canuto2007spectral}
Claudio Canuto, M.~Yousuff Hussaini, Alfio Quarteroni, and Thomas~A. Zang.
\newblock \emph{Spectral Methods: Fundamentals in Single Domains}.
\newblock Springer, 2006.
\newblock \doi{10.1007/978-3-540-30726-6}.

\bibitem[Chen et~al.(2018)Chen, Rubanova, Bettencourt, and
  Duvenaud]{chen2018neural}
Ricky T.~Q. Chen, Yulia Rubanova, Jesse Bettencourt, and David~K Duvenaud.
\newblock Neural ordinary differential equations.
\newblock In S.~Bengio, H.~Wallach, H.~Larochelle, K.~Grauman, N.~Cesa-Bianchi,
  and R.~Garnett (eds.), \emph{Advances in Neural Information Processing
  Systems}, volume~31. Curran Associates, Inc., 2018.
\newblock URL
  \url{https://proceedings.neurips.cc/paper_files/paper/2018/file/69386f6bb1dfed68692a24c8686939b9-Paper.pdf}.

\bibitem[DeepMind et~al.(2020)DeepMind, Babuschkin, Baumli, Bell, Bhupatiraju,
  Bruce, Buchlovsky, Budden, Cai, Clark, et~al.]{deepmind2020optax}
DeepMind, Igor Babuschkin, Kate Baumli, Alison Bell, Surya Bhupatiraju, Jake
  Bruce, Peter Buchlovsky, David Budden, Trevor Cai, Aidan Clark, et~al.
\newblock The {DeepMind} {JAX} ecosystem, 2020.
\newblock URL \url{http://github.com/google-deepmind}.

\bibitem[Gander \& Vandewalle(2007)Gander and Vandewalle]{gander2007analysis}
Martin~J. Gander and Stefan Vandewalle.
\newblock Analysis of the {Parareal} time-parallel time-integration method.
\newblock \emph{SIAM Journal on Scientific Computing}, 29\penalty0
  (2):\penalty0 556--578, 2007.
\newblock \doi{10.1137/05064607X}.

\bibitem[Hochbruck \& Ostermann(2010)Hochbruck and
  Ostermann]{hochbruck2010exponential}
Marlis Hochbruck and Alexander Ostermann.
\newblock Exponential integrators.
\newblock \emph{Acta Numerica}, 19:\penalty0 209–286, 2010.
\newblock \doi{10.1017/S0962492910000048}.

\bibitem[Jagtap \& Karniadakis(2020)Jagtap and Karniadakis]{jagtap2020extended}
Ameya~D. Jagtap and George~Em Karniadakis.
\newblock Extended physics-informed neural networks ({XPINNs}): A generalized
  space-time domain decomposition based deep learning framework for nonlinear
  partial differential equations.
\newblock \emph{Communications in Computational Physics}, 28\penalty0
  (5):\penalty0 2002--2041, 2020.
\newblock \doi{10.4208/cicp.OA-2020-0164}.

\bibitem[Karniadakis et~al.(2021)Karniadakis, Kevrekidis, Lu, Perdikaris, Wang,
  and Yang]{karniadakis2021physics}
George~Em Karniadakis, Ioannis~G. Kevrekidis, Lu~Lu, Paris Perdikaris, Sifan
  Wang, and Liu Yang.
\newblock Physics-informed machine learning.
\newblock \emph{Nature Reviews Physics}, 3:\penalty0 422--440, 2021.
\newblock \doi{10.1038/s42254-021-00314-5}.

\bibitem[Kidger \& Garcia(2021)Kidger and Garcia]{kidger2021equinox}
Patrick Kidger and Cristian Garcia.
\newblock Equinox: neural networks in jax via callable pytrees and filtered
  transformations.
\newblock \emph{arXiv preprint arXiv:2111.00254}, 2021.

\bibitem[Krishnapriyan et~al.(2021)Krishnapriyan, Gholami, Zhe, Kirby, and
  Mahoney]{krishnapriyan2021characterizing}
Aditi Krishnapriyan, Amir Gholami, Shandian Zhe, Robert Kirby, and Michael
  Mahoney.
\newblock Characterizing possible failure modes in physics-informed neural
  networks.
\newblock In M.~Ranzato, A.~Beygelzimer, Y.~Dauphin, P.S. Liang, and J.~Wortman
  Vaughan (eds.), \emph{Advances in Neural Information Processing Systems},
  volume~34, pp.\  26548--26560. Curran Associates, Inc., 2021.
\newblock URL
  \url{https://proceedings.neurips.cc/paper_files/paper/2021/file/df438e5206f31600e6ae4af72f2725f1-Paper.pdf}.

\bibitem[Lagaris et~al.(1998)Lagaris, Likas, and
  Fotiadis]{lagaris1998artificial}
Isaac~E. Lagaris, Aristidis Likas, and Dimitrios~I. Fotiadis.
\newblock Artificial neural networks for solving ordinary and partial
  differential equations.
\newblock \emph{IEEE Transactions on Neural Networks}, 9\penalty0 (5):\penalty0
  987--1000, 1998.
\newblock \doi{10.1109/72.712178}.

\bibitem[Lindel{\"o}f(1894)]{lindelof1894application}
Ernest Lindel{\"o}f.
\newblock Sur l'application des m{\'e}thodes d'approximations successives {\`a}
  l'{\'e}tude des int{\'e}grales r{\'e}elles des {\'e}quations
  diff{\'e}rentielles ordinaires.
\newblock \emph{Journal de math{\'e}matiques pures et appliqu{\'e}es},
  10:\penalty0 117--128, 1894.

\bibitem[Lions et~al.(2001)Lions, Maday, and Turinici]{lions2001parareal}
Jacques-Louis Lions, Yvon Maday, and Gabriel Turinici.
\newblock R{\'e}solution d'{EDP} par un sch{\'e}ma en temps ``parar{\'e}el''.
\newblock \emph{Comptes Rendus de l'Acad{\'e}mie des Sciences - Series I -
  Mathematics}, 332\penalty0 (7):\penalty0 661--668, 2001.
\newblock \doi{10.1016/S0764-4442(00)01793-6}.

\bibitem[Lu et~al.(2021)Lu, Meng, Mao, and Karniadakis]{lu2021deepxde}
Lu~Lu, Xuhui Meng, Zhiping Mao, and George~Em Karniadakis.
\newblock {DeepXDE}: A deep learning library for solving differential
  equations.
\newblock \emph{SIAM Review}, 63\penalty0 (1):\penalty0 208--228, 2021.
\newblock \doi{10.1137/19M1274067}.

\bibitem[McClenny \& Braga-Neto(2023)McClenny and Braga-Neto]{mcclenny2020self}
Levi~D. McClenny and Ulisses~M. Braga-Neto.
\newblock Self-adaptive physics-informed neural networks.
\newblock \emph{Journal of Computational Physics}, 474:\penalty0 111722, 2023.
\newblock \doi{https://doi.org/10.1016/j.jcp.2022.111722}.
\newblock URL
  \url{https://www.sciencedirect.com/science/article/pii/S0021999122007859}.

\bibitem[Pachpatte(1998)]{pachpatte1998inequalities}
B.~G. Pachpatte.
\newblock \emph{Inequalities for Differential and Integral Equations}, volume
  197 of \emph{Mathematics in Science and Engineering}.
\newblock Academic Press, San Diego, 1998.
\newblock ISBN 0-12-543430-8.

\bibitem[Raissi et~al.(2019)Raissi, Perdikaris, and
  Karniadakis]{raissi2019physics}
Maziar Raissi, Paris Perdikaris, and George~E. Karniadakis.
\newblock Physics-informed neural networks: A deep learning framework for
  solving forward and inverse problems involving nonlinear partial differential
  equations.
\newblock \emph{Journal of Computational Physics}, 378:\penalty0 686--707,
  2019.

\bibitem[Shih et~al.(2023)Shih, Belkhale, Ermon, Sadigh, and
  Anari]{shih2023parallel}
Andy Shih, Suneel Belkhale, Stefano Ermon, Dorsa Sadigh, and Nima Anari.
\newblock Parallel sampling of diffusion models.
\newblock \emph{Advances in Neural Information Processing Systems},
  36:\penalty0 4263--4276, 2023.

\bibitem[Shukla et~al.(2021)Shukla, Jagtap, and
  Karniadakis]{shukla2021parallel}
Khemraj Shukla, Ameya~D. Jagtap, and George~Em Karniadakis.
\newblock Parallel physics-informed neural networks via domain decomposition.
\newblock \emph{Journal of Computational Physics}, 447:\penalty0 110683, 2021.
\newblock \doi{10.1016/j.jcp.2021.110683}.

\bibitem[Tancik et~al.(2020)Tancik, Srinivasan, Mildenhall, Fridovich-Keil,
  Raghavan, Singhal, Ramamoorthi, Barron, and Ng]{tancik2020fourier}
Matthew Tancik, Pratul Srinivasan, Ben Mildenhall, Sara Fridovich-Keil, Nithin
  Raghavan, Utkarsh Singhal, Ravi Ramamoorthi, Jonathan Barron, and Ren Ng.
\newblock Fourier features let networks learn high frequency functions in low
  dimensional domains.
\newblock In H.~Larochelle, M.~Ranzato, R.~Hadsell, M.F. Balcan, and H.~Lin
  (eds.), \emph{Advances in Neural Information Processing Systems}, volume~33,
  pp.\  7537--7547. Curran Associates, Inc., 2020.
\newblock URL
  \url{https://proceedings.neurips.cc/paper_files/paper/2020/file/55053683268957697aa39fba6f231c68-Paper.pdf}.

\bibitem[Trefethen(2000)]{trefethen2000spectral}
Lloyd~N. Trefethen.
\newblock \emph{Spectral Methods in MATLAB}.
\newblock SIAM, 2000.
\newblock \doi{10.1137/1.9780898719598}.

\bibitem[Versteeg(1994)]{versteeg1994marmousi}
Roelof Versteeg.
\newblock The {M}armousi experience: Velocity model determination on a
  synthetic complex data set.
\newblock \emph{The Leading Edge}, 13\penalty0 (9):\penalty0 927--936, 1994.

\bibitem[Wang et~al.(2021{\natexlab{a}})Wang, Teng, and
  Perdikaris]{wang2021gradient}
Sifan Wang, Yujun Teng, and Paris Perdikaris.
\newblock Understanding and mitigating gradient flow pathologies in
  physics-informed neural networks.
\newblock \emph{SIAM Journal on Scientific Computing}, 43\penalty0
  (5):\penalty0 A3055--A3081, 2021{\natexlab{a}}.
\newblock \doi{10.1137/20M1318043}.

\bibitem[Wang et~al.(2021{\natexlab{b}})Wang, Wang, and
  Perdikaris]{wang2021eigenvector}
Sifan Wang, Hanwen Wang, and Paris Perdikaris.
\newblock On the eigenvector bias of {F}ourier feature networks: From
  regression to solving multi-scale {PDEs} with physics-informed neural
  networks.
\newblock \emph{Computer Methods in Applied Mechanics and Engineering},
  384:\penalty0 113938, 2021{\natexlab{b}}.
\newblock \doi{10.1016/j.cma.2021.113938}.

\bibitem[Wang et~al.(2022)Wang, Yu, and Perdikaris]{wang2022when}
Sifan Wang, Xinling Yu, and Paris Perdikaris.
\newblock When and why {PINNs} fail to train: A neural tangent kernel
  perspective.
\newblock \emph{Journal of Computational Physics}, 449:\penalty0 110768, 2022.
\newblock \doi{10.1016/j.jcp.2021.110768}.

\bibitem[Wang et~al.(2024{\natexlab{a}})Wang, Li, Chen, and
  Perdikaris]{wang2024piratenets}
Sifan Wang, Bowen Li, Yuhan Chen, and Paris Perdikaris.
\newblock {PirateNets}: Physics-informed deep learning with residual adaptive
  networks.
\newblock \emph{Journal of Machine Learning Research}, 25\penalty0
  (402):\penalty0 1--51, 2024{\natexlab{a}}.

\bibitem[Wang et~al.(2024{\natexlab{b}})Wang, Sankaran, and
  Perdikaris]{wang2022respecting}
Sifan Wang, Shyam Sankaran, and Paris Perdikaris.
\newblock Respecting causality for training physics-informed neural networks.
\newblock \emph{Computer Methods in Applied Mechanics and Engineering},
  421:\penalty0 116813, 2024{\natexlab{b}}.
\newblock \doi{10.1016/j.cma.2024.116813}.

\bibitem[Wong et~al.(2024)Wong, Ooi, Gupta, and Ong]{wong2024fls}
Jian~Cheng Wong, Chin~Chun Ooi, Abhishek Gupta, and Yew-Soon Ong.
\newblock Learning in sinusoidal spaces with physics-informed neural networks.
\newblock \emph{IEEE Transactions on Artificial Intelligence}, 5\penalty0
  (3):\penalty0 985--1000, 2024.
\newblock \doi{10.1109/TAI.2022.3192362}.

\bibitem[Yu et~al.(2022)Yu, Lu, Meng, and Karniadakis]{yu2022gradient}
Jeremy Yu, Lu~Lu, Xuhui Meng, and George~Em Karniadakis.
\newblock Gradient-enhanced physics-informed neural networks for forward and
  inverse {PDE} problems.
\newblock \emph{Computer Methods in Applied Mechanics and Engineering},
  393:\penalty0 114823, 2022.
\newblock \doi{10.1016/j.cma.2022.114823}.

\bibitem[Zhu \& Yin(2026)Zhu and Yin]{zhu2026perturbative}
Xianli Zhu and Jia Yin.
\newblock Perturbative-neusa: A structured spectral framework for
  time-dependent pdes.
\newblock \emph{arXiv preprint arXiv:2607.24345}, 2026.
\newblock \doi{10.48550/arXiv.2607.24345}.

\end{thebibliography}
\end{document}